\documentclass[letterpaper]{article} % DO NOT CHANGE THIS
\usepackage{lipsum}
\usepackage{aaai25}  % DO NOT CHANGE THIS
\usepackage{times}  % DO NOT CHANGE THIS
\usepackage{helvet}  % DO NOT CHANGE THIS
\usepackage{courier}  % DO NOT CHANGE THIS
\usepackage[hyphens]{url}  % DO NOT CHANGE THIS
\usepackage{graphicx} % DO NOT CHANGE THIS
\usepackage{natbib}  % DO NOT CHANGE THIS AND DO NOT ADD ANY OPTIONS TO IT

\usepackage{caption} % DO NOT CHANGE THIS AND DO NOT ADD ANY OPTIONS TO IT
\usepackage{booktabs}
\usepackage{amsthm}
\usepackage{amsmath}
\usepackage{amsfonts}
\usepackage{amssymb}
\usepackage{caption}
\usepackage{subcaption}
\usepackage{tabularray}
\usepackage{float}
\usepackage{tikz}
\usepackage{etoolbox}
\usepackage{makecell}
\usepackage{mathtools}
\AtBeginEnvironment{tblr}{\small}
\usepackage{enumitem}

\theoremstyle{plain}
\newtheorem{theorem}{Theorem}[section]

\newtheorem{lemma}[theorem]{Lemma}

\theoremstyle{definition}
\newtheorem{definition}[theorem]{Definition}

\theoremstyle{remark}

\usepackage{algpseudocodex}
\usepackage{algorithm}

\usepackage[disable]{todonotes}

\usepackage{soul}
\usepackage[most]{tcolorbox}
\usepackage{graphicx}
\usepackage{adjustbox}
\usepackage{seqsplit}

\usepackage[table,xcdraw,dvipsnames]{xcolor}
\definecolor{codegreen}{rgb}{0,0.6,0}        % Dark green (though you said no green earlier, it’s in your original)
\definecolor{codegray}{rgb}{0.5,0.5,0.5}     % Gray for comments or numbers
\definecolor{codepurple}{rgb}{0.58,0,0.82}   % Purple for keywords
\definecolor{backcolour}{rgb}{0.95,0.95,0.92}% Light background

\usepackage{ifthen}

\definecolor{cyan10}{HTML}{E5F6FF}
\definecolor{cyan20}{HTML}{BAE6FF}
\definecolor{cyan60}{HTML}{0072c3}
\definecolor{cyan70}{HTML}{00539a}
\definecolor{teal10}{HTML}{D9FBFB}
\definecolor{teal20}{HTML}{9EF0F0}
\definecolor{teal60}{HTML}{007d79}
\definecolor{orange10}{HTML}{FFF2E8}
\definecolor{orange20}{HTML}{FFD9BE}
\definecolor{orange60}{HTML}{ba4e00}
\definecolor{blue10}{HTML}{EDF5FF}
\definecolor{blue20}{HTML}{D0E2FF}
\definecolor{magenta10}{HTML}{FFF0F7}
\definecolor{magenta20}{HTML}{FFD6E8}
\definecolor{magenta30}{HTML}{ffafd2}
\definecolor{magenta50}{HTML}{ee5396}
\definecolor{magenta60}{HTML}{d02670}
\definecolor{magenta70}{HTML}{9f1853}
\definecolor{purple10}{HTML}{F6F2FF}
\definecolor{purple20}{HTML}{E8DAFF}
\definecolor{purple30}{HTML}{d4bbff}
\definecolor{purple70}{HTML}{8a3ffc}
\definecolor{rose10}{HTML}{FCF2ED}
\definecolor{rose20}{HTML}{F9D9D1}
\definecolor{red10}{HTML}{FFF1F1}
\definecolor{red20}{HTML}{FFD7D9}
\definecolor{green10}{HTML}{DEFBE6}
\definecolor{green20}{HTML}{A7F0BA}
\definecolor{yellow10}{HTML}{fcf4d6}
\definecolor{yellow20}{HTML}{fddc69}
\definecolor{gray20}{HTML}{e0e0e0}
\definecolor{gray30}{HTML}{c6c6c6}
\definecolor{gray40}{HTML}{a8a8a8}
\definecolor{gray80}{HTML}{393939}

\newcommand{\badval}{\cellcolor[HTML]{FFCCC9}}
\newcommand{\goodval}{\cellcolor[HTML]{9AFF99}}
\usepackage{thmtools,thm-restate}

\newcommand{\aref}[1]{Appendix~\ref{#1}}

\newcommand{\ttref}[1]{Theorem~\ref{#1}}
\newcommand{\tpref}[1]{Proposition~\ref{#1}}

\definecolor{nircolor}{RGB}{181,214,167}
\definecolor{trevorcolor}{RGB}{163,193,244}

\usepackage{newfloat}
\usepackage{listings}
\DeclareCaptionStyle{ruled}{labelfont=normalfont,labelsep=colon,strut=off} % DO NOT CHANGE THIS
\floatstyle{ruled}
\newfloat{listing}{tb}{lst}{}
\floatname{listing}{Listing}
\usepackage{hyperref} % -- This package is specifically forbidden
\nocopyright %-- Your paper will not be published if you use this command
\usepackage{cleveref}
\crefname{figure}{Figure}{Figures}
\crefname{table}{Table}{Tables}
\crefname{chapter}{Chapter}{Chapters}
\crefname{section}{Section}{Sections}
\crefname{lstlisting}{Listing}{Listings}

\title{The Dark Side of Rich Rewards: Understanding and Mitigating Noise in VLM Rewards}
\author {
    Sukai Huang\textsuperscript{\rm 1},
    Shu-Wei Liu\textsuperscript{\rm 2},
    Nir Lipovetzky\textsuperscript{\rm 1} and 
    Trevor Cohn\textsuperscript{\rm 1}\thanks{Now at Google DeepMind}
}
\affiliations {
    \textsuperscript{\rm 1}School of Computing and Information Systems, The University of Melbourne, Australia\\
    \textsuperscript{\rm 2}Max Planck Institute for the Physics of Complex Systems, Germany\\
}

\usepackage{bibentry}
\begin{document}
% \linenumbers % # ! REMOVE THIS WHEN SUBMISSION
\maketitle % Overcoming Reward Model Noise in Instruction-Guided Reinforcement Learning

\begin{abstract}
While Vision-Language Models (VLMs) are increasingly used to generate reward signals for training embodied agents to follow instructions, our research reveals that agents guided by VLM rewards often underperform compared to those employing only intrinsic (exploration-driven) rewards, contradicting expectations set by recent work. We hypothesize that false positive rewards -- instances where unintended trajectories are incorrectly rewarded -- are more detrimental than false negatives. We confirmed this hypothesis, revealing that the widely used cosine similarity metric is prone to false positive estimates. To address this, we introduce \textsc{BiMI} (\textbf{Bi}nary \textbf{M}utual \textbf{I}nformation), a novel reward function designed to mitigate noise. \textsc{BiMI} significantly enhances learning efficiency across diverse and challenging embodied navigation environments. Our findings offer a nuanced insight of how different types of reward noise impact agent learning and highlight the importance of addressing multimodal reward signal noise when training embodied agents.
\end{abstract}


\section{Introduction}

\label{sec:chap_3_introduction}

% What the problem is and why it is important to solve it
Natural language instructions are increasingly recognized as a valuable source of reward signals for guiding embodied agents to learn complex tasks. In particular, a growing trend in embodied agent learning involves using vision-language models (VLMs) for reward modeling. This approach measures the semantic similarity -- often quantified by cosine similarity -- between the embedding representations of an agent's behaviors (i.e., past trajectories) and the provided instructions, all within the same embedding space \citep{Kaplan2017BeatingAW, DBLP:conf/ijcai/GoyalNM19, DBLP:conf/corl/GoyalNM20, DBLP:conf/icml/DuWWCDA0A23,DBLP:conf/iclr/RocamondeMNPL24,DBLP:conf/icml/WangSZXBHE24}. 

\begin{figure}
    \centering
    \includegraphics[width=0.5\textwidth]{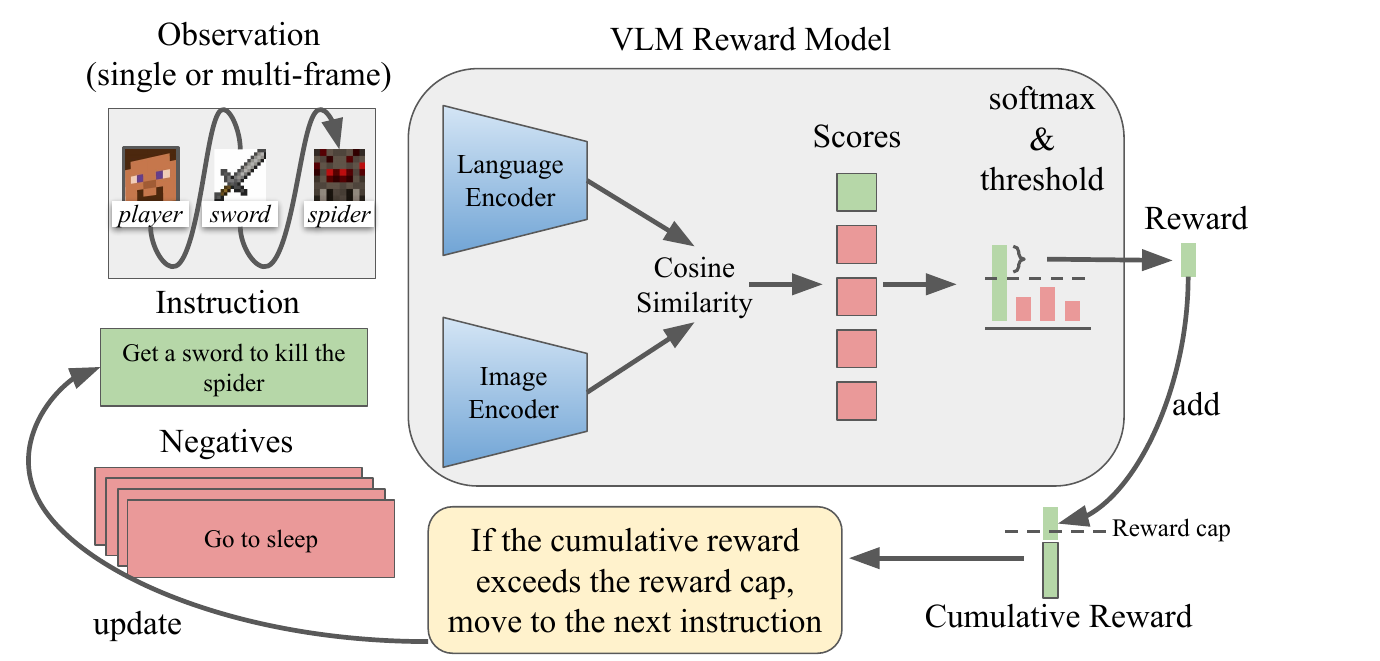}
    \caption{Illustration of embodied RL agents using VLM reward model}
    \label{fig:illustration_of_vlm_reward}
\end{figure}

However, we observed that embodied RL agents trained with VLM rewards, while effective in simplified settings, often struggled with tasks involving complex dynamics and longer action horizons\footnote{Code available at \url{https://shorturl.at/3uVGP}}. This is evident in several recent works -- for instance, \citet{DBLP:conf/ijcai/GoyalNM19} reported the effective use of VLM rewards in Montezuma's Revenge, a notoriously challenging Atari game. However, we observed that this success was confined to individual sub-tasks and the agent struggled when attempting to scale up to the full game. Similarly, \citet{DBLP:conf/icml/DuWWCDA0A23} demonstrated impressive performance of VLM rewards in guiding agents within a 2D survival game. However, their study was conducted in a modified environment with a reduced observation and action space using internal game state information and manually defined macro actions (see \aref{subsec:lit_rev_model_free_rl_LLM_VLM_instruction_reward}). Consequently, when we tested their methods in the original, unmodified environment, we found that their agent's performance did not exceed that of agents using only \emph{intrinsic} (exploration-driven) rewards.

The consistent underperformance of VLM rewards, \textbf{particularly their unexpected failure to outperform intrinsic rewards}, raised concerns about their reliability. This discrepancy, where VLM rewards underperformed contrary to their perceived potential, prompted us to investigate the underlying causes of this performance gap.

Our findings indicate that \textbf{noisy reward estimates} in VLMs are a key factor contributing to poor learning efficacy. Specifically, our analysis centered around two classes of noise: \emph{false positives}, which involve rewarding unintended trajectories and can mislead the agent into reinforcing suboptimal behaviors; and \emph{false negatives}, which occur when correct trajectories are not rewarded, thus providing little feedback and reintroduce the reward sparsity issue. We posit that false positive rewards are not only more prevalent but potentially more detrimental to the learning process than false negatives, a hypothesis supported by our empirical and theoretical findings. Among various sources of reward noise, our study particularly investigates the \textbf{approximation errors} within the commonly used cosine similarity metric. We examine how these errors generate false positive rewards, which in turn hinder learning.

To this end, we propose a novel reward function, \textsc{BiMI} (\textbf{Bi}nary \textbf{M}utual \textbf{I}nformation Reward). It uses binary signals to directly reduce the occurrence of false positives and incorporates a mutual information term to prevent overfitting to noisy signal sources. Our experiments demonstrate that \textsc{BiMI} significantly improves the learning efficacy of instruction-following agents trained with VLM rewards across various complex, \emph{long-horizon} tasks in embodied environments. We validate the effectiveness of \textsc{BiMI} in both \emph{Markovian} and \emph{non-Markovian} reward settings, showcasing its robustness and generalizability across different types of reward models. 

% Summarize contribution of the paper in precise and clear manner

\section{Related Work}
We provide a brief overview of the related work for this study. For detailed ones, please refer to \aref{subsec:lit_rev_model_free_rl_LLM_VLM_instruction_reward}.

\noindent \textbf{VLMs as Reward Models.}\quad Despite the rising use of VLM-based reward models in embodied RL \citep{Kaplan2017BeatingAW, DBLP:conf/ijcai/GoyalNM19, DBLP:conf/corl/GoyalNM20, DBLP:conf/icml/DuWWCDA0A23,DBLP:conf/iclr/RocamondeMNPL24,DBLP:conf/icml/WangSZXBHE24, DBLP:conf/cvpr/WangHcGSWWZ19, Shridhar2021CLIPortWA, DBLP:conf/icml/MahmoudiehPD22}, many works focus on proof-of-concept results in simplified environments, sidestepping the challenges posed by noisy reward signals. Critical issues such as reward noise and its effect on policy convergence remain largely understudied. In particular, some approaches sidestep noise by leveraging internal state information or predefined action macros \citep{DBLP:conf/icml/DuWWCDA0A23, DBLP:conf/nips/WangCCLML23}. However, these methods struggle to scale to real-world applications due to their reliance on environment-specific priors, which lack generalizability in dynamic, unstructured settings.

\noindent \textbf{Pessimistic Value Estimation.}\quad The RL community has demonstrated that pessimism under uncertainty improves robustness in settings like offline RL optimization \citep{uehara2021pessimistic, bai2024pessimistic}. These methods conservatively estimate state values to avoid overoptimism in poorly explored or uncertain regions. However, such techniques are rarely applied to settings with learned VLM-based rewards, \textbf{where noise arises not from environmental uncertainty but from misaligned or overconfident reward signals}. Moreover, these methods remain largely alien to recent studies of VLM reward + RL embodied agent research. Our work bridges this gap by applying pessimism to mitigate learned reward noise, addressing a critical yet overlooked challenge in this emerging field.

\noindent \textbf{Optimistic Exploration in Online RL.}\quad
Pessimistic estimates in RL are mostly an offline phenomenon; online RL instead typically embraces optimism under uncertainty to drive efficient exploration. Classic model-based schemes such as RMax initialize all unknown state-action pairs with the maximum possible reward, thereby provably encouraging the agent to visit under-explored regions of the MDP \citep{brafman2002r}. In the model-free regime, intrinsic reward methods measure the surprise of an agent's actions, such as the prediction error of a forward model \citep{Burda2018ExplorationBR, Wan2023DEIREA}, or the novelty of state visits \citep{Zhang2021NovelDAS}. Our work investigates the integration of pessimistic, VLM-based reward signals with optimistic intrinsic rewards in an online embodied environment.

\noindent \textbf{Reward Signal from Human Preference.}\quad Recent work on RL from Human Feedback (RLHF) \citep{DBLP:conf/nips/Ouyang0JAWMZASR22} also leverage expert preference as a reward signal. However, our work differs in key aspects. Unlike RLHF's focus on textual outputs, our approach involves evaluating cross-modal similarities between visual and textual data in environments requiring long-horizon decision-making and frequent embodied interactions, a domain not typically covered by RLHF.

\noindent \textbf{Mitigating Misspecified Rewards.}\quad Prior works proposed mitigating false positive rewards by training a parallel exploration policy to escape local optima caused by misspecified rewards \citep{Ghosal2022TheEO, icml/FuZ0XB24}. In contrast, we propose a novel reward function that directly penalizes likely false positive reward signals during training. We further show that our method complements exploration-based methods and achieves superior performance when combined.

\section{Formal Problem Statement}
\label{sec:chap_3_background}

We frame our task as an MDP defined by a tuple $\mathcal M = \langle \mathcal S, \mathcal A, \mathcal P, s_0, r^{e}, \gamma \rangle$, where $\mathcal S$ represents a set of states $s \in \mathcal S$, $\mathcal A$ represents a set of actions $a \in \mathcal A$, and $\mathcal P(s' | s, a)$ describes the dynamics of the environment. $s_0 \in \mathcal S$ is the initial state and $\gamma \in (0,1)$ is the reward discount factor. $r^{e}(s, a)$ is the environmental reward function. An agent's trajectory is a sequence of states and actions $\tau_t = \langle s_{0}, a_{0}, \ldots, s_{t} \rangle$. 

In this work, we focus on a sparse reward setting, where the agent receives a $+1$ reward only when reaching goal states $S_G \subset \mathcal{S}$, and 0 otherwise, with $|S_G| \ll |\mathcal{S}|$. This sparse reward setting motivates the use of expert instructions and VLMs to provide auxiliary reward signals for more effective RL. Here, it's crucial to note that by VLM, we specifically mean encoder-based VLMs, as opposed to generative VLMs (see \aref{subsubsec:lit_rev_architecture_training_VLM_architecture}). For conciseness, we will simply use `VLM' to refer to encoder-based VLMs.

Specifically, we have a walkthrough $L$ that breaks down a complex task into $n$ expert-defined sub-tasks, each represented by a natural language instruction that is not necessarily atomic and can encompass multiple finer sub-goals ($L = \{l_1, l_2, \ldots, l_n\}$). By following these sequential instructions, the agent can navigate from the initial state towards the goal states. A dedicated \emph{non-Markovian} VLM-based reward model\footnote{The use of \emph{non-Markovian} reward functions in MDP has been well-established, particularly through the work on reward machines \citep{DBLP:conf/icml/IcarteKVM18, DBLP:conf/aaai/CorazzaGN22}. For a complete evaluation, we also tested \emph{Markovian} version of VLM reward function (Pixl2R) in our experiments.} $r^{v}(\tau_t, l_{m(t)})$ is used to assess how well the agent's trajectory at current time $t$ fulfills an instruction sentence $l_{m(t)}$. For the detailed mechanism and the pseudo-code of the existing VLM reward + RL algorithm, please refer to \aref{ireward_procedure}. Note that our proposed \textsc{BiMI} reward function uses a different mechanism to determine if instruction $l_{m(t)}$ is completed at time $t$ (explained shortly in \cref{sec:preferring-binary-signal}). The VLM provides auxiliary rewards by evaluating the semantic similarity between $\tau_t$ and $l_{m(t)}$, as illustrated in Figure~\ref{fig:illustration_of_vlm_reward}. The use of VLM-based reward model transform the original MDP into a shaped MDP \(\widetilde{\mathcal{M}} = \langle \mathcal{S}, \mathcal{A}, \mathcal{P}, s_0, \tilde{r}, \gamma \rangle\), where the reward function becomes \(\tilde{r} = r^{e}(s, a) + r^{v}(\tau_t, l_{m(t)})\).

% Later in \cref{sec:approximation_error}, we discuss how SoTA VLMs struggle to interpret composite instructions, and we term this limitation ``composition insensitivity''. For instance, consider two instructions: ``pick up the sword after closing the door'' and ``pick up a sword before closing the door''. Despite their opposite temporal requirements, the latent vectors generated by a VLM for these instructions often appear similar, failing to capture the critical ordering of actions.

% As discussed in \cref{subsec:lit_rev_foundation_rl_reward_shaping}, the use of VLM-based reward model transform the original MDP into a shaped MDP \(\widetilde{\mathcal{M}} = \langle \mathcal{S}, \mathcal{A}, \mathcal{P}, s_0, \tilde{r}, \gamma \rangle\), where the reward function becomes \(\tilde{r} = r^{e}(s, a) + r^{v}(\tau_t, l_{m(t)})\).

\section{Theoretical Analysis}
\label{sec:chap_3_theoretical_analysis}
In this work, we define ``reward noise'' as errors made by the VLM-based reward model when evaluating an agent's trajectory: either spuriously rewarding a trajectory that doesn't satisfy the instruction (false positive) or failing to reward one that does (false negative). In this section, we first show that, in the absence of noise, auxiliary rewards derived from language instructions, reflecting progress toward the goal, accelerate convergence compared to relying solely on sparse environmental rewards, $r^e$. We then prove that false positive rewards impede convergence more severely than false negatives under the Heuristic guided RL (HuRL) framework \citep{DBLP:conf/nips/ChengKS21}.

\subsection{Auxiliary Rewards and Convergence}
% due to page limit we cannot put them in the main body
% In sparse reward setting, the gradient landscape is almost 0 everywhere, except for cases where a trajectory reaches a state that is $\delta$-close to a goal state. Without gradient information to guide updates, the policy must perform a random walk in the parameter space. Let $\theta_0$ be the initial parameter vector and $\theta^*$ be the optimal parameter vector whose policy reaches $s \in S_G$.
With sparse rewards, the gradient landscape is nearly flat, making gradient-based updates indistinguishable from a random walk in parameter space (see \tpref{prop:sparse_reward_means_random_walk} in \aref{app_subsec:app_sec:sparse_reward_random_walk}).
In this problem, we are interested in $D:= \|\theta_{goal} - \theta_0\|$ which is the distance in parameter space from initial parameters $\theta_0$ to goal parameters $\theta_{goal}$. We make the following assumption:
\begin{restatable}{assumption}{decompositiontotaldistance}
    \label{assump:decomposition_total_distance}    
    Expert knowledge can guide the parameter search along a path in parameter space, defined by a sequence of $n$ intermediate parameter vectors $\theta_1, \dots, \theta_n$, where each $\theta_i$ represents the parameters after learning sub-task $l_i$. As a result, the overall distance $D$ can be decomposed into segments: $ D \approx \sum_{k=1}^{n-1} d_i$, where $d_i = \|\theta_{i+1} - \theta_{i}\|$.
\end{restatable}
We therefore prove the following proposition:
\begin{restatable}{proposition}{taskdecompositionrandomwalkconvergence}
    \label{prop:task_decomposition_random_walk_convergence}
    The sum of expected time for a series of random walks, each covering the shorter distance of an individual sub-task, is less than the expected time to travel the entire distance $D$ in one long random walk: $\frac{1}{n-1}\mathbb E[T_D] \leq \mathbb E\left[ \sum_{i=1}^{n-1} T_{d_i} \right] < \mathbb E[T_D]$.
    % \begin{gather*}
    %      \frac{1}{n-1}\mathbb E[T_D] \leq \mathbb E\left[ \sum_{i=1}^{n-1} T_{d_i} \right] < \mathbb E[T_D].
    % \end{gather*}
\end{restatable}
\noindent and thus show that subgoal-based auxiliary rewards improve the convergence of random walk optimization in a sparse reward landscape up to a factor of $(n-1)$. See \aref{app_subsec:app_sec:conv_on_sparse_landscape} for the proof. The theoretical takeaway is that auxiliary rewards aligned with intermediate subgoals can accelerate convergence of RL algorithms compared to learning from a single delayed reward at the goal state. In practice, however, this improvement depends critically on the quality of the auxiliary reward signal --- which we examine in the next section.

\subsection{Connection to Heuristic-Guided RL}
The reward noise issue can be framed within the context of HuRL by \citet{DBLP:conf/nips/ChengKS21}. HuRL mandates that auxiliary reward signals serve as heuristics, where $h: \mathcal S \rightarrow \mathbb R$ approximates the future total rewards an agent expects to get starting from state $s$ under the optimal policy $\pi^*$ (i.e., $h(s) \approx V^*(s)$). First of all, we establish the following assumption:

\begin{restatable}{assumption}{instruction_as_landmark}
    \label{assump:instruction_as_landmark}
    The expert instruction sequence $L = {l_1, l_2, \dots, l_n}$ specifies a sequence of state and action landmarks guiding the agent toward goal states $\mathcal{S}_G$. Landmarks are key states or key actions that must hold or occur in any trajectories that reach $\mathcal{S}_G$. 
\end{restatable}

Assumption~\ref{assump:instruction_as_landmark} is reasonable and reflects the quality of expert instructions, which are crucial for the agent to reach the goal efficiently.

We now present the important proposition connecting VLM-based reward models to HuRL:
\begin{restatable}{proposition}{vlmtohurlprop}
    \label{prop:link_vlm_to_hurl}
    In sparse reward settings, VLM-based reward models, as implemented in Algorithm~\ref{alg:instruction_guided_rl}, can be viewed as a heuristic function $h(s_t)$ in HuRL that estimates $V^*(s_t)$.
\end{restatable}

The proof is in \aref{app_subsec:app_sec:vlm_reward_model_heuristic}. HuRL framework allows us to analyze how false positive rewards influence on \emph{optimality gap}, defined as $V^{*}(s_0) - V^{\pi}(s_0)$, a metric that is used for theoretical analysis of the convergence speed of RL algorithms --- the smaller the gap, the faster the convergence. We demonstrate that false positive rewards increase the upper bound of this gap, whereas false negative rewards maintain this upper bound.

\subsection{False Positives and Heuristic Overestimation}
We provide the formal definition of false positive and false negative rewards from both instruction-following (IF) and heuristic perspectives:

\begin{definition}[False Positive Rewards]
    \label{def:false_positive_rewards}
    A false positive reward occurs when: 
    
    \textbf{IF Perspective:} For a trajectory $\tau_t$ that does not satisfy instruction $l_{m(t)}$, the VLM-based reward $r^v(\tau_t, l_{m(t)}) > 0$.
    
    \textbf{Heuristic Perspective:} The heuristic $h(s_t) > V^*(s_t)$, overestimating the optimal value of $s_t$.
\end{definition}

\begin{definition}[False Negative Rewards]
    \label{def:false_negative_rewards}
    A false negative reward occurs when:

    \textbf{IF Perspective:} The VLM-based reward $r^v(\tau_t, l_{m(t)}) \approx 0$ for a trajectory $\tau_t$ that \emph{does} satisfy instruction $l_{m(t)}$.

    \textbf{Heuristic Perspective:} The heuristic $h(s_t) \approx 0 < V^*(s_t)$, underestimating the optimal value of $s_t$.
\end{definition}

\begin{restatable}{proposition}{ifsameasheuristic}
    \label{prop:two_perspective_equivalent}
    False positive rewards in the IF perspective imply false positive rewards in the heuristic perspective.
\end{restatable}
\begin{proof}
    The heuristic is approximated as $h(s_t) \approx V^*(s_t) = 1 \cdot \gamma^{\widetilde{T}-t}$ in \emph{sparse reward setting}, calculated based on an assumed optimal path length $ \widetilde{T}$. Here, $\widetilde{T}-t$ measures the remaining steps towards the goal. When a trajectory $ \tau_t $ receives a high reward but fails to fulfill instruction $ l_{m(t)} $, it corresponds to a high $ h(s_t)$, thus a low $\widetilde{T}-t$. Yet the agent must eventually reattempt $l_{m(t)}$ due to Assumption~\ref{assump:instruction_as_landmark}, $ s_t $ is actually further from the goal than estimated. Therefore, the actual distance $T - t$ to reach the goal will exceed $\widetilde{T} -t$, and $ V^*(s_t)$, calculated with the actual $ T $, will be smaller than $ h(s_t)$, calculated with $ \widetilde{T} $. This thus explains how a false positive reward from an instruction-following (or VLM) perspective leads to an overestimation of the heuristic.
\end{proof}

Researchers have advocated the benefits of pessimistic value estimation to enhance the stability of RL algorithms \citep{kumar2020conservative, jin2021pessimism}. In HuRL, \citet{DBLP:conf/nips/ChengKS21} further identify a beneficial property: when a heuristic is Bellman-consistent pessimism (i.e., $ \max_a (\mathcal B h)(s,a) \geq h(s) $), it results in a capped upper bound on the optimality gap.

\begin{definition}[Bellman-consistent Pessimistic $h$]
    Recall Bellman equation of $h$: $(\mathcal B h)(s,a) = r(s,a) +\gamma \mathbb E_{s' \sim \mathcal P(\cdot | s,a)}[h(s')]$. 
    A heuristic function $h$ is said to be \emph{Bellman-consistent pessimism} with respect to an MDP $\mathcal M$ if $\max_a (\mathcal B h)(s,a) \geq h(s)$. This condition essentially means that the heuristic $h$ never overestimates the true value of a state.
\end{definition}

\begin{restatable}{proposition}{fpviolatepessimisticnature}
    \label{prop:fp_violate_pessimistic_nature}
    Even if the heuristic remains conservative for all successor states, a single \textbf{false positive} overestimation, i.e., $h(s) > V^*(s)$, can violate the pessimistic condition by causing $max_a(\mathcal B h)(s,a) < h(s)$.
\end{restatable}

The proof is in \aref{app_subsec:app_sec:vlm_reward_model_heuristic_pessimistic}. This implies that maintaining a pessimistic heuristic is inherently fragile, as the presence of a false positive in any state disrupts the pessimistic condition.

Our result on the impact of false positive and false negative rewards on convergence is captured as follows:
\begin{restatable}{theorem}{mainresultfalsepositiveincreasesbias}
    \label{theo:main_result_false_positive_increases_bias}
    Adapting the optimality gap analysis from \citet{DBLP:conf/nips/ChengKS21}, the gap $V^*(s_0) - V^\pi(s_0)$ in heuristic-guided RL decomposes into regret and bias terms. False negatives preserve the upper bound on bias, as they maintain the heuristic's pessimism, i.e., $h(s) \leq V^*(s)$. In contrast, false positives breaks the pessimistic property, strictly increase the bias without an upper bound, thereby leading to a larger optimality gap and slower convergence.
\end{restatable}

See Appendix~\ref{app_subsec:detailed_fp_bad_justification} for the formal version and detailed proof. In the next section, we present a case study on cosine similarity metrics, demonstrating how they contribute to false positive rewards in learned reward models.

\section{False Positives From Cosine Similarity}
\label{sec:approximation_error}
% In non-RL domains such as information retrieval and recommendation systems, cosine similarity has been prevalent in measuring the similarity of two entities.
% However, its application in RL faces significant challenges. The sequential nature of decision-making in RL, where state transitions and action order are crucial for task completion, renders 
% \begin{wrapfigure}{r}{0.35\textwidth}
%     \begin{center}
%   \includegraphics[width=0.35\columnwidth]{Figures/chapter_3/figures/illustration_of_false_positives.pdf}
%   \end{center}
%   \caption[Schematic diagram of false positives in embedding space]{Schematic diagram of false positives in embedding space.}
%   \label{fig:false_positive_issue}
% \end{wrapfigure}

\begin{figure}
    \centering
    \includegraphics[width=0.40\textwidth]{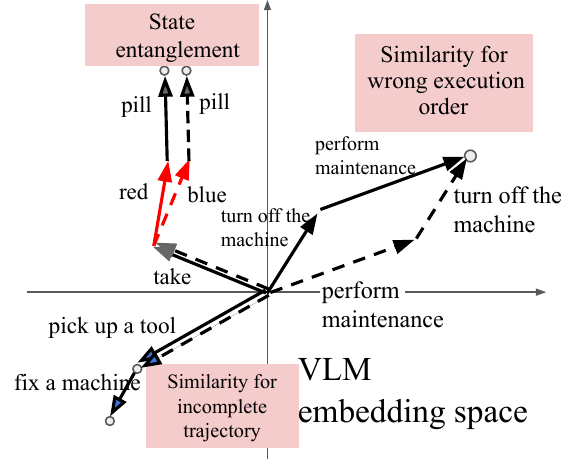}
    \caption[Schematic diagram of false positives in embedding space]{Schematic diagram of false positives in embedding space.}
    \label{fig:false_positive_issue}
\end{figure}
This section identifies and discusses two fundamental issues with cosine similarity scores in sequential decision-making contexts: \emph{state entanglement} and \emph{composition insensitivity}. The former issue, state entanglement, refers to the metric's inability to recognize trajectories that, while being cosine similar to the target instruction in the embedding space, fail to reach the goal states in $S_G$. The latter issue refers to the metric's tendency to reward trajectories even when the temporal relationships between sub-tasks are not satisfied.

% \noindent \textbf{Scope of Discussion.}\quad While prior research, such as the study by \citet{DBLP:conf/icml/DuWWCDA0A23}, has investigated the issue of false negatives in VLM-based reward models -- wherein the model fails to detect the correspondence between the agent's actions and the provided instructions, resulting in a lack of reward and slow policy convergence -- our study shifts the focus to the equally critical but less discussed issue of false positives. False positive rewards occur when agents receive rewards for actions that do not fulfill the intended instructions or task objectives. We posit that these false positives are not only more prevalent but potentially more detrimental to the learning process than false negatives, a hypothesis supported by our empirical findings.
\noindent \textbf{The Issue of State Entanglement}\quad State entanglement refers to the issue where the cosine similarity metric erroneously pays more attention to lexical-level similarity while lacking comprehension of the underlying state transitions. Consequently, rewards are given to trajectory-instruction pairs that are cosine similar in embedding space but in fact result in distinct state transitions. For instance, consider the significant contrast between ``take the \emph{red} pill'' and ``take the \emph{blue} pill''. Despite their lexical similarity, they lead to different states. However, the cosine similarity metric may represent them as similar due to the shared words, disregarding the critical difference in state outcomes (see \cref{fig:false_positive_issue} top left). Understanding state transitions is crucial in sequential decision-making scenarios. Otherwise, rewards may be given to trajectories that lead to unintended states, potentially prolonging the path to the goal state by necessitating corrective actions or re-attempts.

\noindent \textbf{The Issue of Composition Insensitivity}\quad One might wonder why composition insensitivity persists despite using a pointer mechanism (\aref{ireward_procedure}) that enforces sequential instruction order. The issue arises because sentences $l$ in the instruction sequence $L$ are not always \emph{atomic}: a single instruction like ``tidy the room before leaving'' may implicitly encode multiple sub-tasks. While the pointer ensures progress through high-level steps, it does not help the VLM parse the internal structure of complex, non-atomic instructions. As a result, ``composition insensitivity'' emerges, misaligning the agent's trajectory and undermining reward accuracy. Unfortunately, \textbf{even state-of-the-art natural language understanding models struggle with decomposing atomic sentences without losing critical sequential or contextual information} \citep{dziri2024faith}. Under this circumstance, composition insensitivity in cosine similarity metrics gives rise to two issues: (1) \emph{rewarding incomplete task execution} -- cosine similarity may incorrectly reward incomplete task execution by giving high scores even when critical steps are missing. We observed this phenomenon particularly in the \emph{Montezuma} environment, where RL agents tend to \emph{hack} the reward system by focusing on the easiest actions that yield rewards (e.g., moving towards a direction) rather than executing more complex, timely actions. \textbf{Eventually, this leads to an overestimation of the agent's progress towards the ultimate goal} (see \cref{fig:false_positive_issue} bottom left). (2) \emph{insensitivity to the ordering of execution} -- VLM models often fails to adequately penalize incorrect execution sequences, rather, it assigns high rewards based merely on the presence of relevant actions, disregarding their order (see \cref{fig:false_positive_issue} top right). In contrast to some advancements in language models, compact visual and sentence embeddings from multimodal VLMs remain largely insensitive to sequential information \citep{pham2020out}. When the task is order-sensitive, executing actions in the wrong sequence prolongs the path towards the goal state, as agents need to re-attempt the correct order --- yet the VLM's reward may still incentivize the agent to continue with the incorrect sequence. We empirically demonstrate these issues in Section~\ref{sec:experiments_on_existence}, showing their substantial impact on agent learning in sparse reward environments.

% To empirically demonstrate the issue, Section~\ref{sec:experiments_on_existence} presents experiments on these issues and their impact on agent learning in sparse reward environments.

\section{Experiments on Reward Noise Impact}
\label{sec:experiments_on_existence}
% have hypothesis stated here
Our experiments test the following hypothesis: \textbf{(H1)} The two issues of \emph{state entanglement} and \emph{composition insensitivity} exist, \textbf{(H2)} \emph{false positive} rewards are prevalent during training, \textbf{(H3)} VLM reward models lacking noise handling mechanisms underperform against intrinsic reward models in sparse reward environments, \textbf{(H4)} \emph{false negatives} may not be as harmful as \emph{false positives}.

\begin{figure*}[t]
    \centering
    
    \begin{minipage}{0.64\textwidth}
        \centering
        \includegraphics[width=\textwidth]{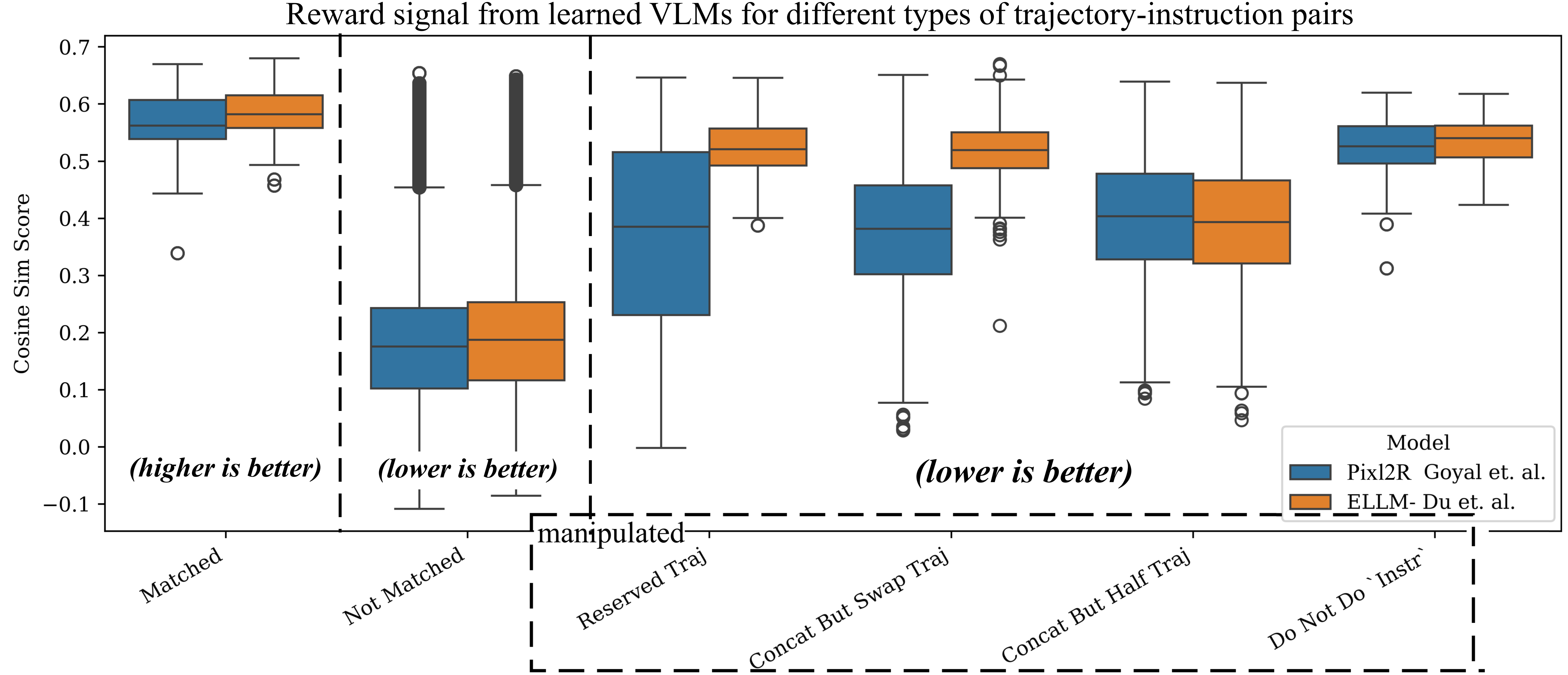}
        \caption{Learned VLM models performed badly with O.O.D. examples. They incorrectly assign high scores to manipulated pairs, which should be low as the trajectories in the manipulated pairs fail the instruction.}
        \label{fig:cosine_sim_score_offline_eval_illu}
    \end{minipage}
    \hspace{0.1cm}
    \begin{minipage}{0.34\textwidth}
        \centering
        \includegraphics[width=\textwidth]{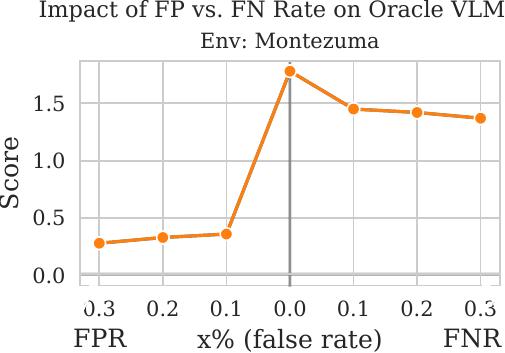}
        \caption{The false positive vs. false negative oracle model. The false positive model get a more severe drop in the final training score.}
        \label{fig:FP_FN_comparison}
    \end{minipage}

\end{figure*}

\begin{figure*}[t]
    \centering
    \begin{minipage}{0.68\textwidth}
        \centering
        \includegraphics[width=\textwidth]{Figures/chapter_3/figures/h2_show_prevalence_of_noisy_rewards.png}
        \caption{The heatmap shows rewards received at various locations, with larger circle sizes indicating higher rewards. The later figures shows the offsets between the state where rewards are given and the actual goal-reaching state. Agents are getting both issues of false positives and false negatives}
        \label{fig:show_prevalence_of_noisy_rewards}
    \end{minipage}
    \hspace{0.1cm}
    \begin{minipage}{0.29\textwidth}
        \centering
        \includegraphics[width=\textwidth]{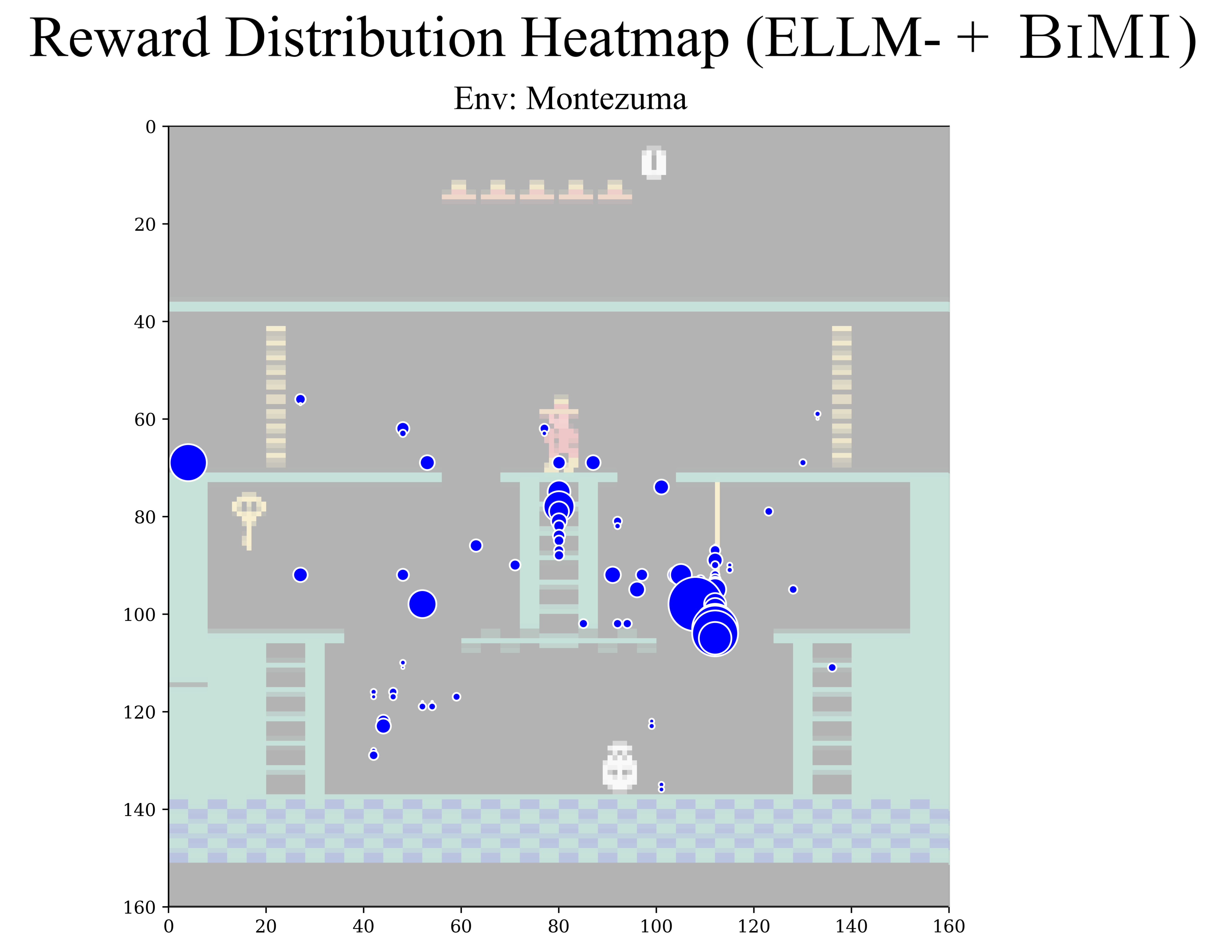}
        \caption{The ratio of false positive rewards is significantly reduced after applying \textsc{BiMI}}
        \label{fig:reward_distribution_heatmap_with_bimi}
    \end{minipage}

\end{figure*}

\textbf{Experimental Setup}\quad We evaluate these hypotheses through various challenging sparse-reward environments: (1) \emph{Crafter}, an open-ended 2D Minecraft \citep{DBLP:conf/iclr/Hafner22}; (2) \emph{Montezuma}, a classic hard adventure game in Atari \citep{DBLP:journals/jair/BellemareNVB13}; and (3) \emph{Minigrid `Go To Seq'}, a hard task involving long-horizon navigation and object interactions \citep{chevalier2018babyai}. Details on why these environments are \textbf{preferred over other test benchmarks}, along with the evaluation metrics, are provided in \aref{app_subsec:app_sec:ant_defense_test_env}. A \textbf{Markovian} and a \textbf{Non-Markovian} reward model were tested: (1) \emph{Pixl2R} by \citet{DBLP:conf/corl/GoyalNM20}, which uses only the current video frame to determine if the goal state specified in the instruction has been reached; and (2) \emph{ELLM-}, a variant of \emph{ELLM} by \citet{DBLP:conf/icml/DuWWCDA0A23} that directly uses preset expert instructions and compares them with the transition differences of the agent's trajectory. The VLM backbones used are: (1) \emph{CLIP} \citep{DBLP:conf/icml/RadfordKHRGASAM21}, pretrained by image-text pairs; and (2) \emph{X-CLIP} \citep{Ma2022XCLIPEM}, pretrained by video-text pairs. To ensure high-quality finetuning data, we used internal information from the game engine to annotate expert trajectories from expert agents. To demonstrate how noisy reward signals hinder learning, we selected a strong intrinsic reward model \emph{DEIR} \citep{Wan2023DEIREA} for comparison. It provides auxiliary rewards based on observation novelty to encourage exploration. The detailed VLM-based reward model procedure, finetuning process and training hyperparameters are in \aref{app:additional_experiment_details}.

\textbf{Reward Noise Issue}\quad To investigate \textbf{H1} --- the existence of two issues, \emph{state entanglement} and \emph{composition insensitivity}, in VLM-based reward models --- we evaluated the models' sensitivity by examining how cosine similarity scores change for manipulated trajectory-instruction pairs. These manipulations, designed to probe to robustness against noise, included the following: 

\begin{enumerate}[leftmargin=1em]
    \item \textbf{Trajectory Reversal:} We inverted the sequence of frames within each trajectory (i.e., \texttt{frames = frame[::-1]}) to assess the model's ability to detect reversed state transitions. This manipulation tests whether the model can distinguish between forward and backward progression in the state transition.
    \item \textbf{Instruction Negation:} We modified the original instructions by adding negation (e.g., changing ``do $l_k$'' to ``do not do $l_k$'' or ``avoid $l_k$''). This tests the model's sensitivity to semantic changes in the instruction that fundamentally alter the goal.
    
    \item \textbf{Concatenation and Order Swapping:} Given two trajectory-instruction pairs $(\tau_1, l_1)$ and $(\tau_2, l_2)$, we created concatenated pairs and then swapped the order in one modality. For example: a) Original concatenation: $(\tau_1 + \tau_2, l_1 + l_2)$; b) Swapped trajectory: $(\tau_2 + \tau_1, l_1 + l_2)$; c) Swapped instruction: $(\tau_1 + \tau_2, l_2 + l_1)$. This tests the model's sensitivity to the order of components in multi-step tasks.
    
    \item \textbf{Concatenation with Partial Content:} We concatenated pairs but truncated one modality. For instance: a) Truncated trajectory: $(\tau_1, l_1 + l_2)$; b) Truncated instruction: $(\tau_1 + \tau_2, l_1)$. This assesses the model's ability to detect partial mismatches in longer sequences.
    
\end{enumerate}

Our results reveal a critical flaw in the reward model: for manipulated pairs that fundamentally fail to fulfill the instruction, the model paradoxically assigns high similarity scores (see Figure~\ref{fig:cosine_sim_score_offline_eval_illu} for overall results and also Appendix~\ref{app_subsec:other_envs_prevalence_of_false_pos} for individual environments). Note that the poor performance in the negation case aligns with broader challenges in natural language processing. Recent studies \citep{hossain2022analysis, truong2023language} have highlighted that negation is central to language understanding but is not properly captured by modern language models. This limitation extends to VLMs and directly leads to false positive rewards.

% # ! Add conceptual graph here to illustrate the false positive issues. 

\textbf{Prevalence of False Positives}\quad To address \textbf{H2}, we analyzed reward distribution heatmap from VLM-based reward models during training. The heatmap (\cref{fig:show_prevalence_of_noisy_rewards} left) revealed a concerning trend: RL agents engage in reward hacking, receiving rewards across vast areas of the environment rather than just at goal states. For instance, in \emph{Montezuma} where the goal is to grab the key and escape the room, we observed that agents received rewards even for falling off cliffs, which undoubtedly contribute to false positive rewards. For environments without fixed camera views, we calculated the step offset between the current rewarded state and the actual goal state. A positive offset indicates a false positive reward, as the reward was given before reaching the goal. Conversely, a negative offset indicates a false negative, where the agent reached the goal but the reward model failed to acknowledge it (see Figure~\ref{fig:show_prevalence_of_noisy_rewards}). Notably, besides positive offsets, we observed a large amount of negative offsets in Minigrid environments. We attribute this to Minigrid's abstract shape-based visual representations, which fall outside the VLM's training distribution.

% \begin{table}[]
%     \centering
%     \caption[
%         Agent performance in sparse reward environments (No noise handling)
%     ]{Score metric across environments (equivalent to total rewards, higher is better). $\star$ denotes baseline intrinsic reward model. VLM reward models without noise handling underperformed. All are based on PPO.}
%     \label{tab:stage_1_rl_results}
%     \resizebox{\columnwidth}{!}{%
%     \begin{tabular}{lccccc}
%     \hline
%     Models        & Type          & \multicolumn{1}{c}{Monte.} & \multicolumn{1}{c}{Minigrid} & \multicolumn{1}{c}{Crafter} & \% vs. DEIR \\ \hline
%     PPO           & Pure          & $0.151$             & $24.9$          & $16.8$          & \badval $-28\%$              \\
%     DEIR $\star$  & Intrinsic     & $0.174$             & $55.5$           & $19.7$          & --                   \\ \hline
%     Pixl2R        & VLM & $0.142$             & $12.4$           & $9.40$           & \badval $\mathbf{-49\%}$     \\
%     ELLM-         & VLM & $0.150$             & $19.4$          & $10.8$          & \badval $-41\%$              \\
%     Pixl2R + DEIR & VLM + intr. & $0.176$             & $17.3$           & $10.4$          & \badval $-38\%$              \\
%     ELLM- + DEIR  & VLM + intr. & $0.178$             & $30.9$           & $11.8$          & \badval $-27\%$             \\ \hline
%     \end{tabular}%
%     }
% \end{table}

\textbf{Impact on Learning}\quad We trained agents using learned VLM reward models and compared 

\begin{table}[t]
    \centering
    % Table here
    \captionof{table}[Agent performance in sparse reward environments (No noise handling)]
    {Score metric across environments (equivalent to total rewards, higher is better). $\star$ denotes baseline intrinsic reward model. VLM reward models without noise handling underperformed. All are based on PPO.}
        \label{tab:stage_1_rl_results}
    \resizebox{\columnwidth}{!}{%
    \begin{tabular}{lccccc}
    \hline
    Models        & Type          & \multicolumn{1}{c}{Monte.} & \multicolumn{1}{c}{Minigrid} & \multicolumn{1}{c}{Crafter} & \% vs. DEIR \\ \hline
    PPO           & Pure          & $0.151$             & $24.9$          & $16.8$          & \badval $-28\%$              \\
    DEIR $\star$  & Intrinsic     & $0.174$             & $55.5$           & $19.7$          & --                   \\ \hline
    Pixl2R        & VLM & $0.142$             & $12.4$           & $9.40$           & \badval $\mathbf{-49\%}$     \\
    ELLM-         & VLM & $0.150$             & $19.4$          & $10.8$          & \badval $-41\%$              \\
    Pixl2R + DEIR & VLM + intr. & $0.176$             & $17.3$           & $10.4$          & \badval $-38\%$              \\
    ELLM- + DEIR  & VLM + intr. & $0.178$             & $30.9$           & $11.8$          & \badval $-27\%$             \\ \hline
    \end{tabular}%
    }
\end{table}
their learning efficacy against intrinsic reward models. As shown in Table~\ref{tab:stage_1_rl_results}, our results confirmed \textbf{H3}: \textit{instruction-following RL agents using learned VLM reward models without noise handling consistently underperform compared to DEIR, the intrinsic reward-based RL agent.}

In our efforts to assess the impact of false positive rewards from auxiliary reward model without the interference of other factors such as domain shift, poor data quality, and errors from other issues such as the choice of multimodal architectures, we devised a \textbf{simulated} reward model that access to internal state information from the game engine. Experiments with these simulated models provide initial evidence of the differential impact of false negatives versus false positives on training outcomes, as posited by \textbf{H4}, suggesting that false positives --- particularly those tied to temporal insensitivity --- more severely degrade final scores. Due to page limits, we move the details to \aref{app_subsec:simulated_oracle_reward_model}.

To compare the effects of reward noise, we designed an oracle Pixl2R model with two variants: one simulating false negatives by randomly removing $x\%$ of true subgoal rewards, and one simulating false positives by adding small one-off rewards (0.1) to $x\%$ of irrelevant states. This design reflects a realistic failure mode of VLMs, where false positive rewards --- despite being only 1/10 the magnitude of true subgoal rewards and sometimes located far from the optimal trajectory --- can still be exploited by the agent through training iterations, a phenomenon known as \emph{reward hacking}. In contrast, false negatives merely reduce the frequency of positive feedback without misleading the agent. The results indicate that false negatives were less detrimental to agent performance than false positives (see Figure~\ref{fig:FP_FN_comparison}). This aligns with our theoretical analysis in \ttref{theo:main_result_false_positive_increases_bias}.

\section{Addressing the False Positive Issue}
% Having established the detrimental effects of false positive rewards, our solution is to strategically manage the trade-off: reducing false positives, even at the cost of a slight increase in false negatives. Note our proposed solution is not specific to cosine similarity's approximation error, which serves as a case study. Rather, it's a broad strategy for mitigating false positive rewards across various reward models, including noisy VLM reward sources.
\textbf{Binary Signal and Conformal Prediction}
\label{sec:preferring-binary-signal}\quad Having established the detrimental effects of false positive rewards, our solution is to strategically manage the trade-off: reducing false positives, even at the cost of a slight increase in false negatives.
First, we propose a reward function that issues a one-time binary reward only when the similarity between the agent's current trajectory and the instruction exceeds a high confidence threshold. This approach contrasts with previous methods, which provide continuous rewards whenever the reward score exceeds a predefined threshold, and continue to do so until reaching a maximum cap. Our method, however, delivers this reward only once. This approach minimizes the likelihood of accumulating false positive rewards while maintaining adherence to Proposition~\ref{prop:link_vlm_to_hurl}.

To achieve this, we introduce a thresholding mechanism using a calibration set of true positive trajectory-instruction pairs. This threshold, denoted as $\hat q$, is set to the empirical quantile of cosine similarity scores at the significance level $1 - \alpha$. Pairs whose similarity scores fall below this threshold $\hat q$ receive no reward. Conversely, pairs exceeding $\hat q$ receive a one-time $+1$ reward:

\begin{equation}
    r^v_{\textsc{Bi}}(\tau, l_k) = \mathbf{1}_{\{p(l_k\mid \tau) \geq \hat q\}}
\end{equation}

It statistically guarantees a high probability (at least $1 - \alpha$) that the true positive pairs are recognized (i.e., above the threshold) while minimizing the average number of mistakes predicting false positives \citep{sadinle2019least}. The threshold calculation is detailed in Algorithm~\ref{alg:chapt_3_threshold_calculation}.

\begin{figure*}[t]
    \centering
\begin{minipage}{0.595\textwidth}
        \captionof{table}{Model score across various environments. $\star$ is the baseline agents with a learned VLM-based reward model to compare with. \textsc{BiMI} significantly improves performance in \emph{Montezuma} and \emph{Minigrid}, while showing mixed results in \emph{Crafter}}
        \label{tab:stage_2_rl_performance}
        \resizebox{\columnwidth}{!}{%
\begin{tabular}{@{}llclclc@{}}
\toprule
Methods              & \multicolumn{1}{c}{Montezuma} & \% vs. $\star$ & \multicolumn{1}{c}{Minigrid} & \% vs. $\star$ & \multicolumn{1}{c}{Crafter} & \% vs. $\star$ \\ \midrule
Pixl2R $\star$       & $0.142 \pm 0.003$             & --                   & $12.4 \pm 2.43$           & --                   & $9.40 \pm 0.022$           & --                   \\
Pixl2R + Bi          & $0.137 \pm 0.009$             & $-3.5\%$             & $31.2 \pm 2.04$           & \goodval $+151\%$              & $10.7 \pm 0.784$          & \goodval $+14\%$               \\
\textbf{Pixl2R + BiMI}        & $0.162 \pm 0.022$             & \goodval $+14\%$               & $37.5 \pm 7.83$           & \goodval $+199\%$              & $7.95 \pm 0.351$           & \badval $-15\%$              \\ \midrule
Pixl2R + DEIR        & $0.176 \pm 0.009$             & $+23\%$               & $17.3 \pm 0.51$           & $+39\%$               & $10.4 \pm 1.015$          & $+10\%$               \\
\textbf{Pixl2R + BiMI + DEIR} & $0.267 \pm 0.016$             & \goodval $+88\%$               & $57.7 \pm 2.15$           & \goodval $+364\%$              & $11.0 \pm 0.190$          & \goodval $+17\%$               \\ \midrule
ELLM- $\star$        & $0.150 \pm 0.004$             & --                   & $19.4 \pm 10.06$          & --                   & $10.8 \pm 1.017$          & --                   \\
ELLM- + Bi           & $0.151 \pm 0.016$             & $+0.6\%$              & $29.7 \pm 1.29$           & \goodval $+53\%$               & $11.1 \pm 0.601$          & \goodval $+3.2\%$              \\
\textbf{ELLM- + BiMI}         & $0.156 \pm 0.014$             & \goodval $+4.0\%$              & $33.6 \pm 3.99$           & \goodval $+74\%$               & $9.42 \pm 0.267$           & \badval $-12\%$              \\ \midrule
ELLM + DEIR          & $0.178 \pm 0.029$             & $+20\%$               & $30.9 \pm 3.50$           & $+59\%$               & $11.8 \pm 1.152$          & $+9.5\%$               \\
\textbf{ELLM- + BiMI + DEIR}  & $0.279 \pm 0.078$             & \goodval $+86\%$             & $56.2 \pm 6.19$           & \goodval $+190\%$              & $13.1 \pm 0.393$          & \goodval $+22\%$               \\ \bottomrule
\end{tabular}%
}
    
\end{minipage}
\hspace{0.1cm}
\begin{minipage}{0.385\textwidth}
    \centering
    \includegraphics[width=0.88\linewidth]{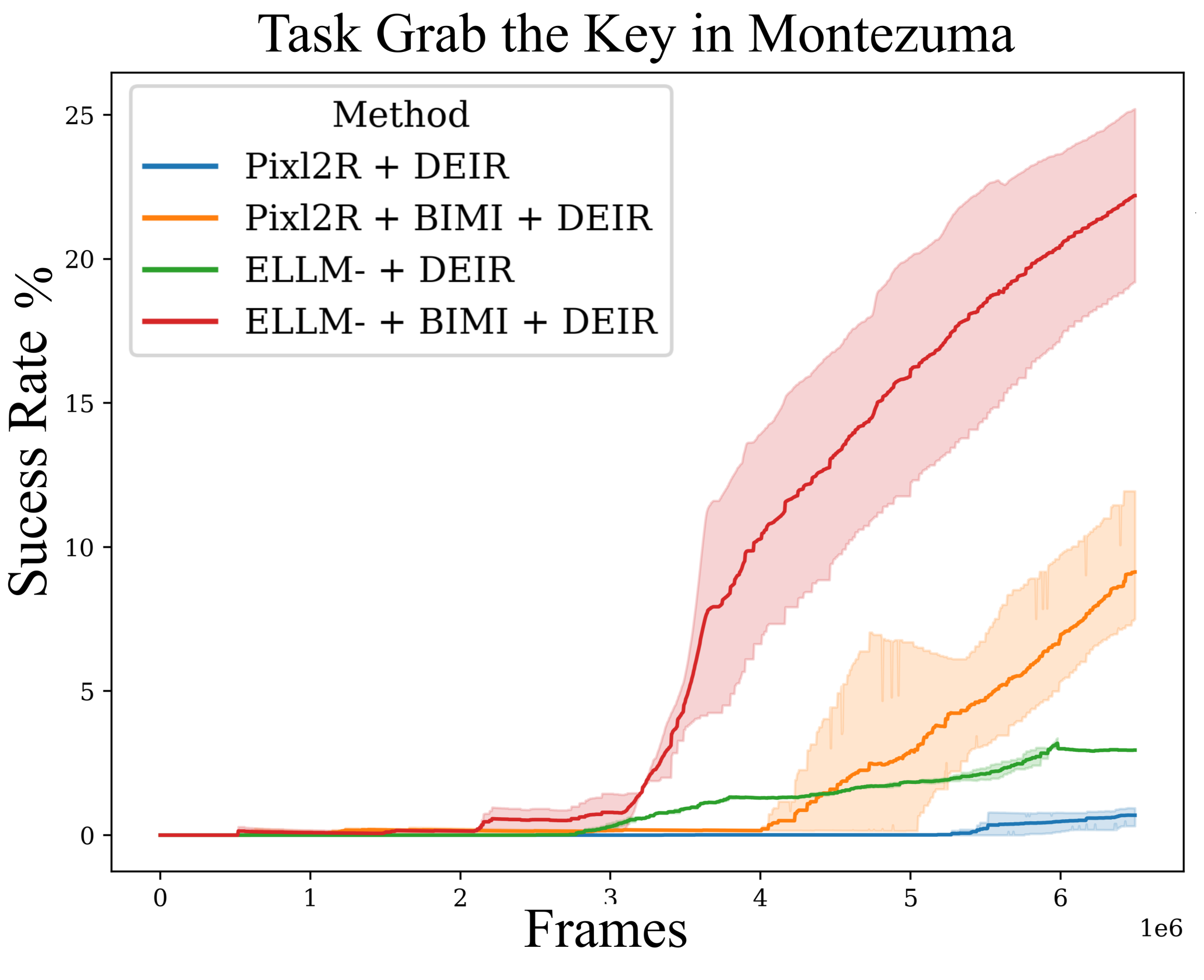}
    \caption{\textsc{BiMI} reward showed faster and higher success rates on difficult tasks in Montezuma} 
    \label{fig:montezuma_difficult_task_success_rate}
\end{minipage}
\end{figure*}

\textbf{Mutual Information Maximization}\quad Intuitively, when we observe rewards coming from a particular signal source too frequently, we tend to downplay the significance of that signal to avoid over-reliance. This intuition is effectively captured by incorporating a \emph{mutual information maximization} term into the reward function. Specifically, the updated reward function $r^v_{\textsc{MI}}(\tau, l_k)$ measures the mutual information between the agent's trajectory and the instruction. Formally, it can be expressed as:
\begin{flalign}
    \label{eq:vlm_reward_model_mi}
    r^v_{\textsc{MI}}(\tau, l_k) & = I(l_k; \tau) = D_{KL}(p(l_k, \tau) \mid\mid p(l_k)p(\tau))\nonumber \\
                & = \mathbb E_{ \tau \sim \pi_{\theta}, l_k \sim L}[\log p (l_k \mid \tau) - \log p(l_k)]
\end{flalign}
\noindent where $\tau = \langle s_{t-W}, a_{t-W}, \ldots, s_t \rangle$ is the agent's trajectory up to current time step $t$, and $W$ is the memory size of the agent for its past trajectory. $p(l_k\mid \tau)$ comes from the similarity score provided by VLMs, referring to the likelihood of the instruction $l_k$ being fulfilled by trajectory $\tau$. $p(l_k)$ is overall likelihood of encountering the instruction $l_k$ in the learning environment. Therefore, the second term in the equation serves as a regularization term that downplays the significance of the reward signal when it is too frequent. For instance, if a VLM frequently detects that the agent's actions are fulfilling the ``climbing the ladder'' instruction, even when the agent is performing unrelated tasks, any reward signal from this instruction will be downplayed. 

We treat the set of instructions $\{l_k\}$ as a finite discrete support and define an empirical frequency:
\begin{flalign}
    s(l_k) \;=\; \mathbb{E}_{\tau \sim \pi_{\theta_{-1}}}\!\Bigl[\tfrac{1}{T_\tau}\sum_{t=1}^{T_\tau}\mathbf{1}\bigl\{p(l_k\mid \tau_t)\ge\hat q\bigr\}\Bigr],
\end{flalign}

\noindent where $\tau_t$ is the agent's trajectory up to time $t$, and we count one whenever the VLM scores $\tau_t$ as fulfilling $l_k$ over the total trajectory length $T_\tau$.  We then normalize to obtain a \emph{probability mass function} (PMF) over instructions: $p(l_k) \;=\;\frac{s(l_k)}{\sum_j s(l_j)}$, and $\sum_j s(l_j)=1$. Because $\{l_k\}$ is discrete, $p(l_k)$ is a valid PMF (not a continuous density) and can be plugged directly into the usual discrete mutual information in \cref{eq:vlm_reward_model_mi}. The subscript $\theta_{-1}$ in $\pi_{\theta_{-1}}$ indicates that the trajectories are sourced from rollouts in the previous policy iteration, acknowledging the impracticality of real-time computation of $p(l_k)$ during an ongoing episode.

To enhance the stability of the training process, we adopt a linearized version of the mutual information maximization approach, as proposed by \citet{li2023internally}. Overall, \textsc{BiMI}, the proposed reward function that enhances the noise resilience of VLM-based reward models, can be expressed as follows:
\begin{equation}
    r^v_{\textsc{BiMI}}(\tau, l_k) = \max(\mathbf{1}_{\{p(l_k\mid \tau) \geq \hat q\}} - p(l_k),0)
\end{equation}
\noindent Note that the \textsc{BiMI} approach primarily mitigates false positives (FPs) rather than false negatives (FN). Both \textsc{Bi} and \textsc{MI} aim to reduce the likelihood of rewarding unintended trajectories, thus addressing the FP issue. While this conservative approach may increase false negatives by pruning some correct trajectories, this trade-off is beneficial, as demonstrated by both our empirical and theoretical results.

\section{Experiments and Results}
\label{sec:chapt_3_experiments}
% We continue to evaluate using \textbf{Markovian} \emph{Pixl2R} and \textbf{Non-Markovian} \emph{ELLM-} and their \textsc{BiMI}-enhanced counterparts, while also exploring potential synergies with the intrinsic reward model DEIR. 
We follow the same experimental setup as in Section~\ref{sec:experiments_on_existence} to ensure a fair comparison. Furthermore, we set confidence level for empirical quantile calculation to be $1 - \alpha=0.9$. We adhered to the standard requirement of limiting the training budget to 1 million frames \citep{DBLP:conf/iclr/Hafner22}.
% for Minigrid and Crafter environments. This constraint poses a significant challenge, particularly in sparse reward settings, as it demands that agents both explore efficiently and exploit their knowledge effectively within this limited budget.

% To achieve the 1 million frame budget, we used the following configuration: 
% \begin{itemize}
%     \item \texttt{nproc}: 8 (Number of processes used for parallel environments)
%     \item \texttt{nstep}: 512 (Length of the rollout stored in the buffer)
%     \item \texttt{nepoch}: 250 (Number of epochs to train the RL policy)
% \end{itemize} 

% This configuration results in approximately 1 million steps: 250 epochs $\times$ 512 steps $\times$ 8 processes $=$ 1,024,000 frames.

% In the case of Montezuma's Revenge, we found that the 1 million frame limit was insufficient due to the game's complexity and sparse reward structure. To address this, we extended the training budget to 8 million frames. It's important to note that even with this increased frame count, agents were still unable to fully solve the task. For context, prior work by \citet{Zhang2021NovelDAS} suggests that 1 billion frames are required to truly master Montezuma's Revenge. This vast difference in required training time underscores the exceptional difficulty of Montezuma's Revenge as a sparse reward task.

\textbf{Montezuma} \quad Pixl2R$+$\textsc{BiMI} demonstrated 14\% performance increase compared to the original models (see Table~\ref{tab:stage_2_rl_performance}), which is slightly below our expectations. We attribute this result to \textsc{BiMI}'s intentional strategy of providing less frequent discrete rewards. While this strategy effectively reduces FPs, it does not substantially mitigate the inherent reward sparsity issue in \emph{Montezuma}. However, \textbf{we discovered a remarkable synergy between \textsc{BiMI} and intrinsic reward models.} While previous models showed no significant improvements with \emph{DEIR} (the intrinsic reward model) alone, combining \textsc{BiMI} and \emph{DEIR} led to a 65\% performance gain. The gap in collaboration effectiveness can be attributed to two factors. In the previous setup, the consistent presence of false positive rewards misled agents towards unacceptable behaviors and hindered further exploration. Now, \text{BiMI}'s less frequent but more meaningful rewards provide anchor points for the agent's learning. Meanwhile, \emph{DEIR}'s intrinsic rewards fill the gaps between these anchor points, encouraging the agent to explore efficiently in the interim.

See \cref{fig:show_prevalence_of_noisy_rewards} left and \cref{fig:reward_distribution_heatmap_with_bimi} for a \textbf{quantitative analysis}: \textsc{BiMI} rewards are now concentrated on key locations. A significant improvement is the minimal rewards given for falling off cliffs, which was a common source of FPs in the original model. Figure~\ref{fig:montezuma_difficult_task_success_rate} demonstrates a higher success rate in grabbing the key in the first room, one of the most difficult tasks in \emph{Montezuma}, highlighting the effectiveness of the proposed reward function and its synergy with intrinsic reward models in guiding agents to solve difficult sparse-reward tasks. 

\noindent\textbf{Minigrid:}\quad ELLM-$+$\textsc{BiMI} achieved a remarkable 74\% improvement in performance compared to the original models.This substantial gain is particularly noteworthy given the unique challenges presented by \emph{Minigrid}. The abstract, shape-based visuals diverge drastically from the natural images used in VLMs' pretraining, preventing the models from effectively utilizing their prior pretraining knowledge. Consequently, VLMs struggled to accurately assess similarities between Minigrid's abstract visuals and textual instructions, resulting in highly noisy reward signals. The significant improvement demonstrated by \textsc{BiMI} underscores its effectiveness in handling noisy signals, directly addressing our primary research challenge. This capability is crucial for deploying instruction-following agents in real-world, unfamiliar scenarios, where visual inputs often deviate from the VLMs' training distribution, leading to noisy reward signals.

\noindent\textbf{Crafter}:\quad We observed an intriguing pattern of results. The \textsc{Bi} component alone led to 14\% and 3.2\% improvement in performance over the original models. However, contrary to our observations in other environments, the addition of the \text{MI} component actually decreased this improvement. This unexpected outcome can be attributed to the unique nature of \emph{Crafter} task, where agents must repeatedly achieve the same subtasks (e.g., drinking water) for survival. The \textsc{MI} component, designed to discourage over-reliance on frequently occurring signals, inadvertently penalized the necessary repetition of survival-critical actions. Nevertheless, \textsc{Bi} alone still managed to improve performance over vanilla VLM-based reward models, suggesting that reducing FPs is still beneficial across all testing environments. The combination of \textsc{BiMI} with \emph{DEIR} (the intrinsic reward model) also showed promising results, indicating a productive balance between exploration (driven by \emph{DEIR}) and exploitation (guided by \textsc{BiMI} instruction reward).

\textbf{Overall Performance and Ablation Study}\quad The \textsc{BiMI} reward function yielded significant performance gains, as detailed in \cref{tab:stage_2_rl_performance}. For the Markovian \emph{Pixl2R} model, \textsc{BiMI} improved scores by 67\%, while the non-Markovian \emph{ELLM-} model improved by 22\%. These improvements, alongside synergy with intrinsic rewards, are depicted in \cref{fig:main_result_lineplot_score}, which illustrates how combining \textsc{BiMI} with intrinsic rewards (i.e., DEIR) boosts agent performance across different environments.

Our ablation study further demonstrates the distinct contributions of the binary reward (\textsc{Bi}) and mutual information (\textsc{MI}) components within the \textsc{BiMI} framework with results shown in \cref{fig:ablation_bi_bimi_lineplot_score}. The \textsc{Bi} mechanism alone drove a 36.5\% performance increase over baseline models. Excluding the \emph{Crafter} environment, the \textsc{MI} component added a further 23\% improvement over \textsc{Bi} alone. Together, these results underscore the importance of both components in mitigating false positive rewards and enhancing agent performance.

\section{Conclusion and Discussion}

Most existing work on VLM-based reward models for embodied instruction-following agents remains at the proof-of-concept stage in simplified environments \citep{DBLP:conf/icml/DuWWCDA0A23,DBLP:conf/iclr/RocamondeMNPL24,DBLP:conf/icml/WangSZXBHE24}, overlooking the challenge of noisy rewards in complex, long-horizon tasks. Our work addresses this gap through theoretical analysis grounded in \emph{heuristic admissibility} --- a classical concept from the automated planning domain but rarely applied to embodied learning agents. This perspective enables us to rigorously explain why VLM-based reward models are prone to failure without explicit noise handling. Building on this foundation, we identify two key insights: (1) false positive rewards are substantially more harmful to policy learning than false negatives, and (2) the proposed \textsc{BiMI} reward function, which incorporates pessimistic rewarding, effectively mitigates this problem. Our findings are supported across three challenging embodied tasks, spanning Markovian and non-Markovian reward models.

\textbf{Limitations} \quad We primarily focused on linear sequences of language instructions, excluding more complex cases. Future research should investigate conditional and ambiguous instructions, which likely introduce additional challenges for VLM-based reward models. There is also a gap in providing a rigorous theoretical foundation for why our theoretical findings extend to non-Markovian reward models.
However, with advancements in deep RL, the distinction between non-Markovian and Markovian models has become increasingly blurred \citep{DBLP:conf/aaaifs/HausknechtS15}.

\bibliography{aaai25}

\newpage
\onecolumn
\appendix
\setlist[enumerate]{leftmargin=1em}
\setlist[itemize]{leftmargin=1em}

% a big appendix title 
\begin{center}
    \textbf{\huge Technical Appendix}
\end{center}

\section*{Appendix Table of Contents}
\begin{enumerate}
    \item \textbf{Anticipatory defense on the choice of test environments and metrics} (\aref{app_subsec:app_sec:ant_defense_test_env}) \\
    Justification for chosen test environments and evaluation metrics.
    
    \item \textbf{Detailed Related Work} (\aref{subsec:lit_rev_model_free_rl_LLM_VLM_instruction_reward}) \\
    An expanded discussion of related research in the domain and research gaps.
    
    \item \textbf{The Complete Procedure of the VLM-based Reward Model} (\aref{ireward_procedure}) \\
    The complete procedure of involving VLM-based reward model in the RL training loop.
    
    \item \textbf{Convergence time on a sparse-reward landscape} (\aref{app_subsec:app_sec:conv_on_sparse_landscape}) \\
    Proof of convergence time on a sparse-reward landscape.
    
    \item \textbf{VLM-based reward model is a heuristic function} (\aref{app_subsec:app_sec:vlm_reward_model_heuristic}) \\
    Proof that VLM-based reward model is a heuristic function.
    
    \item \textbf{False Positive and the Violation of Pessimistic Property of Heuristics} (\aref{app_subsec:app_sec:vlm_reward_model_heuristic_pessimistic}) \\
    Proof that false positive rewards violate the pessimistic property of heuristics.
    
    \item \textbf{False Positive gives unbounded bias in the \emph{optimality gap}} (\aref{app_subsec:detailed_fp_bad_justification}) \\
    Proof that false positive rewards give unbounded bias in the \emph{optimality gap}.
    
    \item \textbf{Additional Experimental Setup Details} (\aref{app:additional_experiment_details}) \\
    Supplementary information about experimental configurations.
    
    \item \textbf{Details of Showing the Prevalence of False Positives in VLM Cosine Similarity Scores} (\aref{app_subsec:other_envs_prevalence_of_false_pos}) \\
    Extra figures for stage 1 experiments. 
    
    \item \textbf{Extra Experiment on Simulated Oracle Reward Model to demonstrate \textbf{H3} and \textbf{H4}} (\aref{app_subsec:simulated_oracle_reward_model}) \\
    Supporting experiments for hypothesis H3 and H4 using simulated oracle reward models.

    \item \textbf{Algorithm for Empirical Quantile Calculation} (\aref{app_subsec:empirical_quantile_calculation}) \\
    Algorithm for calculating empirical quantile of cosine similarity scores.

    \item \textbf{Ablation Study Lineplots for BiMI reward} (\aref{app_subsec:ablation_study_lineplots_bimi}) \\
    Lineplots for ablation study of \textsc{BiMI} reward function.
    
    \item \textbf{Encoder-based Vision-Language Models Explanation} (\aref{subsubsec:lit_rev_architecture_training_VLM_architecture}) \\
    A brief overview of two types of VLMs, encoder-based and generative VLMs.

    \item \textbf{Sparse Reward and Random Walk} (\aref{app_subsec:app_sec:sparse_reward_random_walk}) \\
    Proposition and proof of the relationship between RL convergence in sparse reward setting and random walk.
\end{enumerate}

\section{Anticipatory defense on the choice of test environments and metrics}
\label{app_subsec:app_sec:ant_defense_test_env}
As discussed in \aref{subsec:lit_rev_model_free_rl_LLM_VLM_instruction_reward}, we need test environments that clearly demonstrate the negative impact of noisy reward signals from VLM-based reward models. This requires complex state spaces (both large and diverse) where the approximation errors of VLMs become evident --- specifically, where they struggle to provide accurate alignment scores between an agent's trajectory and given instructions due to the complexity of visual observations.  Additionally, we need longer-horizon tasks where the impact of noisy reward signals can accumulate over time. In contrast, many embodied environments used in previous studies are not suitable due to their short-horizon tasks and simple state spaces. \cref{table:chap_2_2_rl_benchmarks} lists common RL embodied environments used in the literature, explaining why certain environments are inadequate for our evaluation purposes while highlighting suitable environments we aim to use.

\begin{table}[H]
    \centering
    \caption[RL embodied environments for evaluating noisy reward signals from VLM-based reward models.]{Common RL embodied environments and their suitability for evaluating noisy reward signals from VLM-based reward models.}
    \label{table:chap_2_2_rl_benchmarks}
    \resizebox{\linewidth}{!}{%
    \begin{tblr}{
      width = \linewidth,
      colspec = {Q[135]Q[135]Q[163]Q[112]Q[373]},
      cells = {c,b},
      cell{2}{1} = {c=5}{0.939\linewidth},
      cell{8}{1} = {c=5}{0.939\linewidth},
      hlines,
      hline{1,12} = {-}{0.08em},
    }
    \textbf{Env}               & \textbf{Type}         & \textbf{State Space Complexity} & \textbf{Horizon Length} & \textbf{Instruction Source}                                            \\
    \textbf{Not Suitable Envs} &                       &                                 &                         &                                                                             \\
    CartPole \citep{brockman2016openaigym}                  & Classic Control       & Simple                          & Short                   & Manual annotation                                                       \\
    MountainCar \citep{brockman2016openaigym}               & Classic Control       & Simple                          & Short                   & Manual annotation                                                       \\
    Humanoid \citep{todorov2012mujoco}                  & Bipedal robot control & Complex                         & Short                   & Manual annotation                                                       \\
    MetaWorld \citep{DBLP:conf/corl/YuQHJHFL19}                 & Robotic arm           & Moderate~                       & Short                   & Predefined task descriptions                                 \\
    SoftGym \citep{corl2020softgym}                   & Robotic arm           & Moderate                        & Short                   & Predefined task descriptions                                       \\
    \textbf{Suitable Envs}     &                       &                                 &                         &                                                                             \\
    Montezuma \citep{DBLP:journals/jair/BellemareNVB13}                  & Platformer   & Complex                         & Long                    & Manual annotation                                                       \\
    Crafter \citep{DBLP:conf/iclr/Hafner22}                    & Open world survival   & Complex                         & Long                    & Programmable (via engine data) \\
    Minigrid \citep{DBLP:conf/nips/Chevalier-Boisvert23}                  & Grid-world maze       & Controllable                    & Controllable            & Predefined task descriptions                                       
    \end{tblr}
    }
    \end{table}

We describe each testing environment used in our experiments. More details introduction can be found in on the official project homepage of each benchmark \citep{DBLP:conf/iclr/Hafner22, DBLP:journals/jair/BellemareNVB13, DBLP:conf/nips/Chevalier-Boisvert23}. 
\begin{itemize}
    \item \textbf{Crafter} features randomly generated 2D worlds where the player needs to forage for food and water, find shelter to sleep, defend against monsters, collect materials, and build tools. The original Crafter environment does not have a clear goal trajectory or instructions; agents are aimed at surviving as long as possible and exploring the environment to unlock new crafting recipes. We modified the environment to include a preset linear sequence of instructions to guide the agent to mine diamond. However, this instruction was found to hinder the agent's performance. The nature of the task requires dynamic strategies and real-time decision-making, but the fixed instructions limited the agent. For example, the instruction did not account for what to do when the agent is attacked by zombies.
    \item \textbf{Montezuma's Revenge} is a classic adventure platform game where the player must navigate through a series of rooms to collect treasures and keys. The game is known for its sparse rewards and challenging exploration requirements. We manually annotate 97 instructions for the agent to follow, guiding it to conquer the game. The instructions were designed to guide the agent through the game's key challenges, such as avoiding enemies, collecting keys, and unlocking doors.
    \item \textbf{Minigrid `Go to seq' Task}: We use the `Go to seq' task in the Minigrid environment, where the agent must navigate through a sequence of rooms and touch target objects in the correct order. This is a sparse reward task where the agent receives a reward of 1 only upon completing the entire sequence correctly. During the training phase, we randomly generate 50 different tasks, each with a room size of 5, 3 rows, and 3 columns. Each task features a unique room layout and target object sequence. The instruction complexity is set to 3, meaning there are at least 3 target objects to interact with in a specific order.
\end{itemize}

\begin{figure}[H]
    \centering
    \includegraphics[width=0.6\columnwidth]{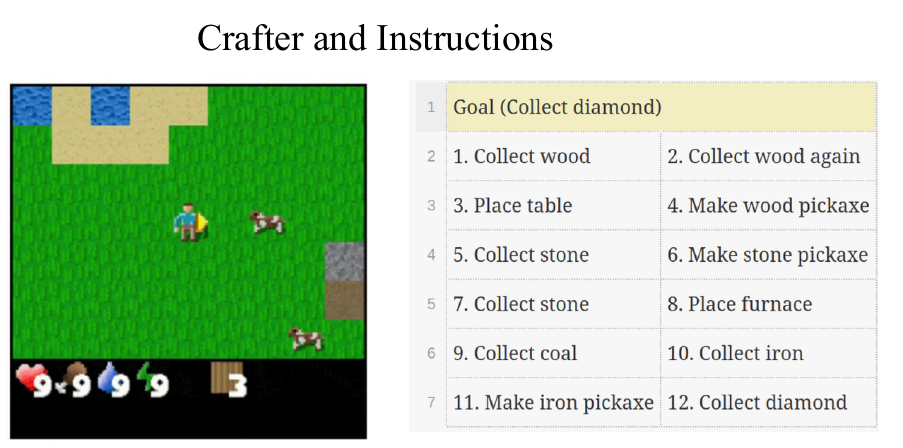}
    \caption[Illustration of the Crafter task]{Illustration of the Crafter task. The agent must survive as long as possible and explore for new crafting recipes.}
\end{figure}

\begin{figure}[H]
    \centering
    \includegraphics[width=0.6\columnwidth]{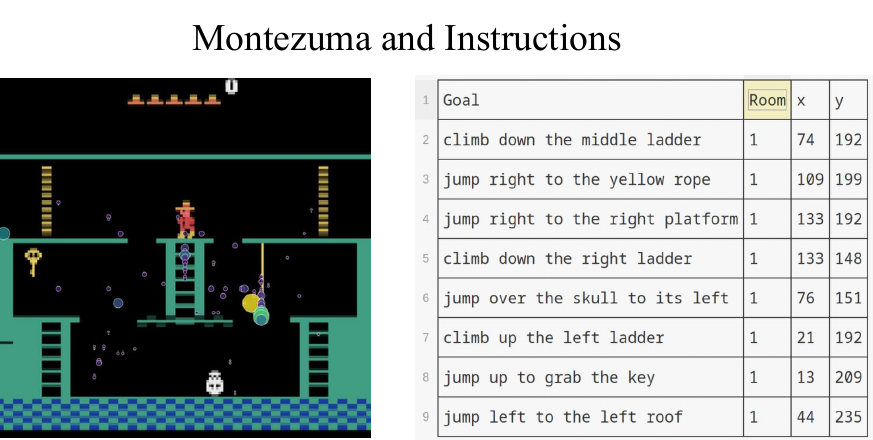}
    \caption[Illustration of the Montezuma's Revenge task]{Illustration of the Montezuma's Revenge task. The agent must navigate through a series of rooms to collect treasures and keys.}
\end{figure}

\begin{figure}[H]
    \centering
    \includegraphics[width=0.6\columnwidth]{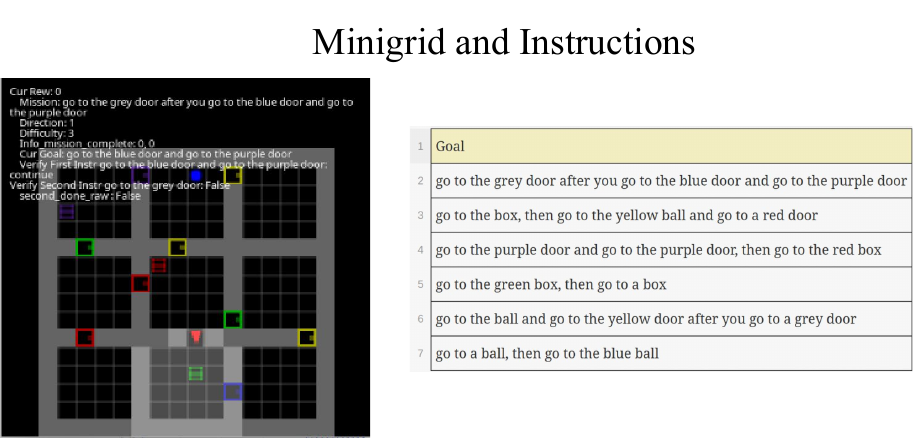}
    \caption[Illustration of the Minigrid `Go to seq' task]{Illustration of the Minigrid `Go to seq' task. The agent must navigate through a sequence of rooms and touch target objects in the correct order.}
\end{figure}

\subsection{Metrics}
\label{subsubsec:lit_rev_evaluation_metrics_benchmarks_rl_metrics}
In our experiments, we used a score metric adapted from the Crafter benchmark \citep{DBLP:conf/iclr/Hafner22} to evaluate agent performance across different environments. This score metric aggregates success rates for individual subtasks using a geometric mean. Formally, the score metric is defined as follows:
\begin{equation}
    \text{Score} = \exp \left( \frac{1}{N} \sum_{k=1}^{N} \ln (1 + s_k) \right) - 1 
\end{equation}
\noindent where $s_k$ is the agent's success rate of achieving instruction $l_k$, and $N$ is the total number of instructions.
 
This metric was chosen over the \emph{maximum total rewards} metric for several reasons:
\begin{enumerate}
    \item \textbf{Consistency in Sparse Reward Settings:} Sparse reward environments often pose significant challenges for reinforcement learning agents. An agent might occasionally achieve high rewards by chance in one rollout but fail to replicate this success consistently in subsequent rollouts. This variability can lead to misleading evaluations if only the maximum total rewards are considered. The Score metric, by measuring the success rate of achieving each subgoal, provides a more stable and consistent measure of an agent's performance.
    \item \textbf{Capturing Learning Stability:} The Score metric evaluates the agent's ability to consistently reproduce successful behaviors across multiple episodes. This is crucial in sparse reward settings, where the agent's performance can fluctuate significantly. By focusing on the success rates of individual subtasks, the Score metric offers a more granular and reliable assessment of the agent's learning progress and stability.
    \item \textbf{Crafter Benchmark Standard:} The Crafter benchmark, which introduces the Score metric, is a well-regarded standard.
\end{enumerate}

Crafter codebase provides \emph{score} metric calculation by default. For Minigrid and Montezuma environments, we use the internal information from the game engine to detect whether the subtasks are completed, thus facilitating the calculation of the \emph{score} metric. 

To evaluate the \textbf{learning speed} of our RL algorithms, we utilize the Area Under the Learning Curve (AUC) metric. This metric, previously employed in the literature \citep{DBLP:conf/ijcai/GoyalNM19,DBLP:conf/corl/GoyalNM20}, quantifies the cumulative performance of the agent throughout the training process, effectively indicating how rapidly the agent improves its policy. We implement a win cap, limiting the maximum number of wins an agent can achieve during training to a certain number (e.g., 5000). Training automatically terminates once this limit is reached, under the assumption that policy learning has \textbf{stabilized} by this point. The AUC is therefore number of wins normalized by the total number of training episodes. Formally, this is expressed as:
\begin{flalign}
\text{AUC} = \frac{\sum_{i=0}^{T} \text{wins}_i}{T} \nonumber
\end{flalign}
where $\text{wins}_i$ is the number of wins at episode $i$, and $T$ is the total number of training episodes.

\section{Detailed Related Work}
\label{subsec:lit_rev_model_free_rl_LLM_VLM_instruction_reward}

The approach of converting natural language instructions into reward signals for RL agents has been a longstanding area of exploration. Early studies \citep{Kaplan2017BeatingAW, DBLP:conf/ijcai/GoyalNM19, DBLP:conf/cvpr/WangHcGSWWZ19, DBLP:conf/corl/GoyalNM20} trained separate text and vision encoders from scratch so as to convert multimodal data into continuous feature vectors, then used a discriminator to assess alignment between trajectories and instructions. The cosine similarity metric, found by \citet{DBLP:conf/nips/FromeCSBDRM13} as an effective measure of semantic alignment, became foundational to these methods. Strikingly, the core architecture of VLM-based reward models, as illustrated in \cref{fig:illustration_of_vlm_reward}, has remained largely unchanged throughout the years. Both pioneering efforts \citep{Kaplan2017BeatingAW, DBLP:conf/ijcai/GoyalNM19, DBLP:conf/cvpr/WangHcGSWWZ19, DBLP:conf/corl/GoyalNM20} and modern implementations continue relying on computing cosine similarity scores between semantic embeddings of instructions and agent trajectories to generate rewards. Formally, given two embedding vectors, \(E(l)\) and \(E(\tau)\), which denote the semantic embeddings of an instruction \(l\) and a trajectory \(\tau\), respectively, the cosine similarity is defined as:

\begin{equation}
    \text{Cosine Similarity}(E(l), E(\tau)) = \frac{E(l) \cdot E(\tau)}{\|E(l)\| \|E(\tau)\|}
\end{equation}

However, before the advent of CLIP \citep{DBLP:conf/icml/RadfordKHRGASAM21} in 2021, these efforts were hampered by the absence of powerful pre-trained encoder-based VLMs. Consequently, researchers focused on proof-of-concept studies, testing their approaches on simpler, short-horizon tasks or sub-tasks --- for example, \citet{DBLP:conf/ijcai/GoyalNM19} evaluated their method on sub-tasks rather than the full Montezuma's Revenge game\footnote{available at \url{https://www.retrogames.cz/play_124-Atari2600.php}}. This likely reflected an understanding that complex, long-horizon tasks were impractical at the time. As such, the field had yet to grapple with noisy reward signals, which remained minimal in these simpler environments.

CLIP \citep{DBLP:conf/icml/RadfordKHRGASAM21}, an encoder-based VLM, marked a turning point by providing a high-quality joint semantic embeddings for vision and language data. Leveraging CLIP, researchers started addressing more complex environments. For example, \citet{DBLP:conf/icml/MahmoudiehPD22} used a CLIP-based reward model to guide a robotic arm in manipulation tasks, demonstrating its feasibility for real-world applications.

In addition to the topics mentioned in the main body of the paper, we continue to highlight related work on reward signals and mitigation strategies below.

\noindent \textbf{Reward Signal from Human Preference.}\quad Recent work on RL from Human Feedback (RLHF) \citep{DBLP:conf/nips/Ouyang0JAWMZASR22} also leverage expert preference as a reward signal. However, our work differs in key aspects. Unlike RLHF's focus on textual outputs, our approach involves evaluating cross-modal similarities between visual and textual data in environments requiring long-horizon decision-making and frequent embodied interactions, a domain not typically covered by RLHF.

\noindent \textbf{Mitigating Misspecified Rewards.}\quad Prior works proposed mitigating false positive rewards by training a parallel exploration policy to escape local optima caused by misspecified rewards \citep{Ghosal2022TheEO, icml/FuZ0XB24}. In contrast, we propose a novel reward function that directly penalizes likely false positive reward signals during training. We further show that our method complements exploration-based methods and achieves superior performance when combined.

\noindent \textbf{Research Gap}\quad Successes with CLIP-based reward models have been particularly pronounced in environments with \textbf{short-horizon} tasks. For instance, \citet{DBLP:conf/iclr/RocamondeMNPL24} achieved promising results in classic control tasks such as CartPole, MountainCar \citep{brockman2016openaigym}, and Humanoid \citep{todorov2012mujoco}, environments that lack long-horizon dependencies. Similarly, \citet{DBLP:conf/icml/WangSZXBHE24} used the SoftGym environment \citep{corl2020softgym}, which includes a suite of short-horizon robotic arm manipulation tasks, as well as MetaWorld environment \citep{DBLP:conf/corl/YuQHJHFL19}, which lacks strict ordering constraints on subgoals, eventually leading to a lack of long-horizon dependencies. 

Importantly, VLM reward models are learned-based reward models and, as such, inherently suffer from noisy reward signals. Sadly, the tendency to focus on simpler environments, or to sidestep the challenges of noisy reward signals, keeps appear in recent work. For example, \citet{chan2023visionlanguage} explored CLIP-based rewards to reduce dependence on human-engineered reward functions, but their experiments were confined to simple housekeeping tasks with small action and state spaces --- far less complex than pixel-based environments like Montezuma's Revenge. Similarly, ELLM \citep{DBLP:conf/icml/DuWWCDA0A23} relied on a \textbf{hard-coded oracle reward model} (labeled as a ``VLM reward model'' for conceptual demonstration) to maintain training stability. Notably, without this oracle, their agent performed worse than a pure RL agent without any reward shaping --- highlighting the limitations imposed by noisy VLM-based rewards in complex environments.

Despite the growing interest of using VLM-based reward models in embodied RL, the literature has largely overlooked the detrimental effects of noisy reward signals on learning efficiency. Key questions remain underexplored: how reliable are these models in complex, long-horizon environments, how do false positive and false negative reward noise affect learning speed, and how to maintain effective learning under noisy reward conditions. These gaps highlight a critical area for progress. In this work, we aim to address this gap by examining these issues in depth and proposing solutions to enhance the applicability of language-guided RL in real-world scenarios.

\section{The Complete Procedure of the VLM-based Reward Model}
\label{ireward_procedure}

Language instructions in real-world tasks are rarely singular; they typically form a sequence that guides an agent step-by-step. A VLM-based reward model must therefore include a mechanism to transition between instructions as the agent progresses. Existing VLM-reward implementations applied a pointer mechanism to decide which instruction the agent should consider at the current step. We follow this implementation and denote the pointer as $m(t)$. It indicates the instruction that the agent is trying to complete at time step $t$. $m(t)$ is updated according to the following rule: 
\begin{equation}
    m(t+1) = \begin{cases}
    1 & \text{if $t = 0$} \\
    m(t) + 1 & \text{if instr. $l_{m(t)}$ completed at $t$} \\
    m(t) & \text{otherwise}
    \end{cases}
\end{equation}
The pointer remains on the current instruction until the accumulated reward for completing $l_{m(t)}$ reaches a predetermined threshold.

Given a sequence of instruction sentences $L = \{l_1, l_2, \dots, l_n\}$, a typical VLM-based reward model, as seen in Pixl2R \citep{DBLP:conf/corl/GoyalNM20} and ELLM \citep{DBLP:conf/icml/DuWWCDA0A23}, maintains this \textbf{pointer} to the current instruction, starting with $l_1$. To track progress, it imposes a maximum reward cap --- for example, 2.0 --- on each instruction. Once the cumulative reward reaches this cap, the pointer advances to the next instruction in the sequence. This approach ensures the agent is incentivized to complete one subtask before moving on to the next.

The complete procedure for training an RL agent with a VLM-based reward model is detailed in \cref{alg:instruction_guided_rl}. This algorithm builds on concepts including MDP formulation, RL algorithms, and the VLM-based reward model (\aref{subsec:lit_rev_model_free_rl_LLM_VLM_instruction_reward}).

\begin{algorithm}[H]
    \caption{Instruction-following RL training with VLM-based reward model}
    \label{alg:instruction_guided_rl}
    \begin{algorithmic}[1]
    \State Initialize policy network $\pi_\theta$
    \State Initialize value network $Q_\phi$
    \State Setup VLM-based reward model $E(\cdot)$
    \State Split instruction essay into sentences $\{l_1, l_2, ..., l_K\}$
    \State Initialize instruction pointer $p = 1$
    \State Initialize cumulative VLM reward $r_{\text{cum}} = 0$
    \State Initialize cumulative VLM reward threshold $q$
    \State Initialize replay buffer $\mathcal{D}$
    \State Initialize agent trajectory memory queue $\tau$ with length $W$
    
    \For{each episode}
        \State Initialize state $s_0$
        \For{$t = 0$ to $T-1$}
            \State Select action $a_t \sim \pi_\theta(a_t|s_t)$
            \State Execute $a_t$, observe next state $s_{t+1}$ and environmental reward $r^e_t$
            \State Enqueue $(s_t, a_t, r_t, s_{t+1})$ in $\tau$
        
            \LComment{Compute VLM reward}:
            \State $r^v_t = p(l_p | \tau) = \frac{E(\tau) \cdot E(l_p)}{\|E(\tau)\| \|E(l_p)\|}$ 
            
            \State Combine rewards: $r_t = r^e_t + (1 -\beta) \gamma r^v_t$ \Comment{$\beta$ is a scaling factor}
            
            \State Store $(s_t, a_t, r_t, s_{t+1})$ in $\mathcal{D}$

            \State $r_{\text{cum}} \gets r_{\text{cum}} + r^v_t$
            \If{$r_{\text{cum}} \geq q$}
                \State $p \gets \min(p + 1, K)$
                \State $r_{\text{cum}} \gets 0$
            \EndIf
            
            \If{Reach Update Frequency}
                \State Sample mini-batch $\{(s_j, a_j, r_j, s_{j+1})\}$ from $\mathcal{D}$
                
                \LComment{Compute TD errors}:
                \State $\delta_j = r_j + \gamma Q_\phi(s_{j+1}, a_{j+1}) - Q_\phi(s_j, a_j)$
                
                \LComment{Update value network}:
                \State $\phi \gets \phi + \alpha_v \sum_j \delta_j \nabla_\phi Q_\phi(s_j, a_j)$
                
                \LComment{Update policy network}:
                \State $\theta \gets \theta + \alpha_p \sum_j Q(s_j, a_j) \nabla_\theta \log \pi_\theta(a_j|s_j)$

            \EndIf
        \EndFor
    \EndFor
    \end{algorithmic}
\end{algorithm}

\section{Convergence time on a sparse-reward landscape}
\label{app_subsec:app_sec:conv_on_sparse_landscape}
In this work, \cref{sec:chap_3_theoretical_analysis}, we have shown that the convergence time of a policy learning algorithm on a sparse-reward landscape can be characterized by the distance in the parameter space. Here, we provide a detailed proof of the convergence time on a sparse-reward landscape.
We can pin down a characteristic convergence time on a sparse-reward landscape. The sparse-reward setting enforces that
\begin{align}
    r^e(s_t) = \begin{cases}
    1 & \text{if $s_t \in \mathcal{S}_G$} \\
    0 & \text{otherwise}
    \end{cases}.
\end{align}
It is clear from the definition that a good trajectory must always be a goal trajectory, and the cumulative reward is simply the reward at the final goal state so $G^e(\tau) = \gamma^T$ in this setting. For a randomly initialized policy, it is highly unlikely that the initial distribution of trajectories contains any goal trajectory due to the sparsity of goal states. The optimization of $V^e_{\pi_\theta}$ thus consists of two parts. The first part is to search for a goal trajectory. The gradient landscape is almost 0 everywhere, except for cases where a trajectory is $\delta$-close to a goal trajectory. Here $\delta$ is the differential unit in the numerical differentiation used in the gradient calculation
\begin{align}
    \theta = \theta + \alpha \nabla_\theta V_{\pi_\theta}
\end{align}
such that $\delta$-close means being numerically accessible within a distance of $|\delta|$ in parameter space. And the second part is to reduce $T$ so that goal trajectories become good trajectories and consequently achieving acceptable policies. For the first part, searching for a trajectory for the target is effectively a random walk in the $d$-dimensional parameter space due to the flat gradient landscape.\\
\begin{restatable}{lemma}{togoalstatetFmestep}
    \label{lemma:to_goal_state_time_step}
    For a random walk in $n$-dimensional space, the expected number of steps $T_D$ needed to travel a distance of $D$ scales with $D^2$.
\end{restatable}
\begin{proof} 
Let $\vec{X_1}$, $\vec{X_2}$, ... $\vec{X_T}$ be IID random unit vectors uniformly distributed on a $(d-1)$-dimensional sphere $\mathcal{S}^{d-1} \subset \mathbb{R}^d$, where $\vec{X_i}=(X_{i1},...,X_{id})$ and $|\vec{X_i}|^2=\sum_{j=1}^d X^2_{ij}=1$. Let $\vec{S_T}:=\sum_{i=1}^T \vec{X_i}$. By a $n$-dimensional cosine rule we have
\begin{align}
    |\vec{S}_T|^2 &= |\vec{S}_{T-1}|^2 + 2\vec{S}_{T-1}\cdot \vec{X}_T + |\vec{X}_T|^2,
\end{align}
and because $\mathbb{E}[\vec{X}_T]=\vec{0}$ 
\begin{align}
    \mathbb{E}[|\vec{S}_T|^2] &= \mathbb{E}[|\vec{S}_{T-1}|^2] + \mathbb{E}[2\vec{S}_{T-1}\cdot \vec{X}_T] + 1\\
    &=\mathbb{E}[|\vec{S}_{T-1}|^2)]+ 2\vec{S}_{T-1}\mathbb{E}[\vec{X}_T] + 1\\
    &=\mathbb{E}[|\vec{S}_{T-1}|^2] + 1\\
    &=\mathbb{E}[|\vec{S}_{T-2}|^2] + 1 +1\\
    &\vdots\\
    &=T
\end{align}
\textit{i.e.} $\mathbb{E}[|\vec{S}_T|]\sim \sqrt{T}$. When a policy is randomly initialized with $\theta_0$, the distance $D$ to a goal policy, $D:=|| \theta_{goal} - \theta_0||$ is fixed and is the distance the random walk needs to travel ($\mathbb{E}[|\vec{S}_T|] = D \sim \sqrt{T_D}$), so the characteristic time needed to travel this distance $T_D \sim D^2$ as we have shown above.
\end{proof}

\decompositiontotaldistance*
Auxiliary rewards essentially open the path for a divide-and-conquer approach by introducing intermediate rewards in the learning process. We introduce BiMI rewards as
\begin{equation}
    r^v_{\textsc{BiMI}}(\tau, l_k) = \max(\mathbf{1}_{\{p(l_k\mid \tau) \geq \hat q\}} - p(l_k),0)
\end{equation}
and the cumulative reward of a single trajectory in the presence of BiMI rewards becomes
\begin{align}
    G^v_{\textsc{BiMI}}(\tau) = \sum_{t=0}^{T} \gamma^t r^v_{\textsc{BiMI}}(\tau, l_{m(t)}).
\end{align}
Because $r^v_{\textsc{BiMI}}$ is either 1$-p(l_k)$ or 0, it effectively breaks the entire task into $n$ segments of sub-tasks $\{ l_1, l_2, ..., l_n \}$ and each sub-task is a sparse-reward problem. Because this decomposition is based on expert knowledge, we can reasonably assume that the start-finish distance $D$ in parameter space is partitioned into $D\approx d_1+d_2+...+d_{n-1}$ without incurring much detour. 

\taskdecompositionrandomwalkconvergence*
\begin{proof}
The expected time taken for each of the sub-tasks then scales with $~d_i^2$ respectively (Lemma~\ref{lemma:to_goal_state_time_step}), and we have
\begin{align}
    d_1^2+d_2^2+...+d_{n-1}^2 &< (d_1+d_2+...+d_{n-1})^2\\
    \mathbb{E}\left[\sum_{i=1}^{n-1}T_{d_i}\right] &< \mathbb{E}[T_D]
\end{align}
because $d_i>0 \forall i$. We can also work out that the upper bound for this improvement is a factor of $n-1$ by invoking the Cauchy–Schwarz inequality $\left(\sum_{i=1}^n u_i^2 \right)\left(\sum_{i=1}^n v_i^2 \right) \geq \left(\sum_{i=1}^n u_iv_i \right)^2$:
\begin{align}
        (d_1^2+d_2^2+...+d_{n-1}^2)(1^2+1^2+...+1^2)&\geq (d_1+d_2+...+d_{n-1})^2\\
    (d_1^2+d_2^2+...+d_{n-1}^2) &\geq \frac{1}{n-1}D^2\\
    \mathbb{E}\left[\sum_{i=1}^{n-1}T_{d_i}\right] \geq \frac{1}{n-1}\mathbb{E}[T_D]
\end{align}
and the equality sign holds (indicating maximal improvement) when $d_1=d_2=...=d_{n-1}$. The intuitive interpretation is that the divide-and-conquer approach is the most effective when the task is divided evenly into subtasks.
\end{proof}

\section{VLM-based reward model is a heuristic function}
\label{app_subsec:app_sec:vlm_reward_model_heuristic}

\vlmtohurlprop*

\begin{proof}
    Consider a sparse reward environment where the external reward is defined as:
    \[
    r^e(s_t) = \begin{cases} 
    1 & \text{if } s_t \in \mathcal{S}_G, \\ 
    0 & \text{otherwise},
    \end{cases}
    \]
    with discount factor $\gamma \in (0, 1)$. For an optimal policy $\pi^*$, the value function at state $s_t$ is:
    \[
    V^*(s_t) = \gamma^{\widetilde{T} - t},
    \]
    where $\widetilde{T}$ is the time to reach a goal state $s_g \in \mathcal{S}_G$ along the shortest path from $s_0$, and $\widetilde{T} - t$ is the remaining steps from $s_t$. This follows because the agent receives a single reward of 1 at $s_g$, discounted back to $s_t$.
    
    In Algorithm~\ref{alg:instruction_guided_rl}, the reward $r^v(\tau_t, l_{m(t)}) \in [0, 1]$ measures alignment between the agent's transition $(\tau_{t-1}, s_t, a_t)$ and the expert instruction $l_{m(t)}$ for sub-task $m(t)$. By design, $r^v = 1$ when the transition fully matches $l_{m(t)}$, and $r^v < 1$ otherwise, implying \emph{partial alignment}.
    
    We assume expert instructions $\{l_1, \dots, l_n\}$ define a trajectory $\tau^*$ that reaches $\mathcal{S}_G$ optimally or near-optimally. If the agent fully follows $l_{m(t)}$ at each step, its path aligns with $\tau^*$, and $\widetilde{T}$ is minimized (say, $\widetilde{T}^*$).  When the agent perfectly follows the instruction at each step along $\tau^*$, the VLM reward is maximized, i.e., $r^v(\tau_t, l_{m(t)}) = 1$.  Crucially, moving along the optimal trajectory $\tau^*$ leads to the optimal value function $V^*(s_t) = \gamma^{\widetilde{T}^* - t}$.

    Now, if the agent deviates from $l_{m(t)}$, it takes a suboptimal action, increasing the expected path length to $\mathcal{S}_G$. Let $\widetilde{T}' > \widetilde{T}^*$ denote this detour. Then, $V^*(s_t) = \gamma^{\widetilde{T}' - t} < \gamma^{\widetilde{T}^* - t}$. Since $r^v$ decreases with misalignment (e.g., from 1 to some $0 \leq r < 1$), it exhibits a \textbf{monotonic relationship} with the value function: higher alignment (larger $r^v$) corresponds to movement towards a shorter path and thus a larger $V^*(s_t)$, and vice versa.
    
    Thus, VLM-based rewards, under Algorithm~\ref{alg:instruction_guided_rl}, function as a heuristic $h(s_t)$ that approximates the relative ordering of states in terms of their optimal value $V^*(s)$.
\end{proof}

\section{False Positive and the Violation of Pessimistic Property of Heuristics}
\label{app_subsec:app_sec:vlm_reward_model_heuristic_pessimistic}

\fpviolatepessimisticnature*

\begin{proof}
    We want to show that given:
    \begin{enumerate}
        \item for all successor states $s'$: $h(s') \leq V^*(s')$ for an arbitrary $s$.
        \item false positive at current arbitrary state $s$: $V^*(s) < h(s)$.
    \end{enumerate}
    The objective is to show that under the above conditions, we obtain $\max_a(\mathcal B h) (s, a) < h(s)$.
    
    \begin{enumerate}
        \item Express the Bellman equation for $h$: $(\mathcal{B} h)(s, a) = R(s, a) + \gamma \sum_{s'} P(s'|s, a) h(s')$
        \item Express the Bellman version of $V^*$: $V^*(s) = \max_{a \in \mathcal{A}} \left[ R(s, a) + \gamma \sum_{s'} P(s'|s, a) V^*(s') \right]$. 
        
        \noindent Given that $ h(s') \leq V^*(s') $ for all $s'$, we can expand it to $\sum_{s'} \mathcal P(s'|s, a) h(s') \leq \sum_{s'} \mathcal P(s'|s, a) V^*(s')$. Therefore, we have:
        \begin{equation}
            R(s, a) + \gamma \sum_{s'} \mathcal P(s'|s, a) h(s') \leq R(s, a) + \gamma \sum_{s'} \mathcal P(s'|s, a) V^*(s') = Q^*(s, a)
        \end{equation}
        \item Taking the maximum over actions to both sides, we have:
        \begin{flalign}
            \max_a\left(R(s, a) + \gamma \sum_{s'} \mathcal P(s'|s, a) h(s')\right) &\leq max_a Q^*(s,a)\\
            \max_a(\mathcal B h)(s,a) &\leq V^*(s)
        \end{flalign}
        \item Apply false positive condition: Given $V^*(s) < h(s)$, substituting into the above inequality, we have $\max_a(\mathcal B h)(s,a) < h(s)$
    \end{enumerate}
    
    \noindent \textbf{Implication}:\quad Maintaining a pessimistic heuristic is inherently fragile because the introduction of a false positive in any state disrupts the pessimistic condition.
\end{proof}

\section{False Positive gives unbounded bias in the \emph{optimality gap}}
\label{app_subsec:detailed_fp_bad_justification}
In this work, \cref{sec:chap_3_theoretical_analysis}, we analyze the impact of false positive rewards on the convergence speed of policy learning based on the optimality gap analysis by \citet{DBLP:conf/nips/ChengKS21}. Here, we provide a detailed proof of the impact of false positive rewards on the convergence speed of policy learning.

We have defined the formal definition of false positive and false negative from both instruction-following (IF) and heuristic perspectives in \cref{sec:chap_3_theoretical_analysis}.

\mainresultfalsepositiveincreasesbias*
We analyze the convergence speed of policy learning by examining the optimality gap, which is defined as the difference between the optimal value of the initial state $s_0$, $V^*(s_0)$, and the value of the initial state under an arbitrary policy $\pi$, $V^\pi(s_0)$. Specifically, we focus on deriving an upper bound for this optimality gap. The key intuition is that a smaller upper bound implies faster convergence to the optimal policy, as fewer iterations of policy updates will be required to reach the optimum.

We begin by stating the theorem made by HuRL authors:
\begin{theorem}[Optimality Gap Decomposition \citep{DBLP:conf/nips/ChengKS21}]
    \label{theo:hurl_performance_decomposition}
    For any policy $\pi$, heuristic $h:\mathcal S\rightarrow \mathbb R$, and mixing coefficient $\beta \in [0,1]$,
    \begin{equation}
        V^*(s_0) - V^\pi(s_0) = \mathrm{Regret}(h,\beta,\pi) + \mathrm{Bias}(h,\beta,\pi)
    \end{equation}
    \noindent where the regret and the bias term are expressed as follows:
    \begin{flalign}
        \mathrm{Regret}(h,\beta,\pi) &\coloneqq
        \beta \left( \widetilde{V}^*(s_0) - \widetilde{V}^\pi(s_0) \right)  + \frac{1-\beta}{1-\gamma}  \left( \widetilde{V}^*(d^{\pi}) - \widetilde{V}^\pi(d^{\pi}) \right)
        \label{eq:regret} \\
        \mathrm{Bias}(h,\beta,\pi) &\coloneqq
        \left( V^*(s_0) - \widetilde{V}^*(s_0) \right) + \frac{\gamma(1-\beta)}{1-\gamma} \mathbb E_{s,a\sim d^\pi}   \mathbb E_{s' \sim \mathcal P(\cdot | s, a)} \left[ h(s')  - \widetilde{V}^*(s')\right]  \label{eq:bias} 
    \end{flalign}
\end{theorem}

We will not provide a proof for Theorem~\ref{theo:hurl_performance_decomposition} in this paper. Please refer to \citep{DBLP:conf/nips/ChengKS21} for details. The theorem demonstrates that the performance gap can be elegantly decomposed into two components: a regret term and a bias term such that:
\begin{enumerate}
    \item The regret term quantifies the difference between $\widetilde{V^*}$ and $\widetilde{V^\pi}$, representing the error caused by $\pi$ being suboptimal in the reshaped MDP $\widetilde{\mathcal M}$. Since $\pi$ is trained by our selected RL algorithm directly on the reshaped MDP $\widetilde{\mathcal M}$, the primary responsibility for minimizing this regret term falls to the RL algorithm itself, not to the design of the auxiliary reward signal. Thus, when evaluating the effects of false positive or false negative rewards, we choose not to focus on bounding the regret term.
    \item The bias term captures two key discrepancies: first, between the true optimal value function $V^*$ of the original MDP $\mathcal M$ and the optimal value function $\widetilde{V}^*$ of the reshaped MDP $\widetilde{\mathcal M}$; second, between $\widetilde{V}^*$ and the heuristic $h(s)$. This term, therefore, reflects the error introduced by addressing the reshaped MDP instead of the original one, alongside how well the heuristic $h$ approximates the optimal value function in the reshaped MDP. Consequently, this bias term directly relates to the quality of the heuristic reward signal $h$. Hence, we focus on analyzing its upper bound to assess the impact of false positive or false negative rewards on this heuristic.
\end{enumerate}

\citet{DBLP:conf/nips/ChengKS21} have further derived the upper bound for the bias term, which we present here as a lemma. We omit the proof for brevity; for detailed proof, please refer to \citep{DBLP:conf/nips/ChengKS21}.

\begin{lemma}[Upper Bound of the Bias Term]
    \begin{flalign}
        \mathrm{Bias}(h,\beta,\pi) \leq (1-\beta)\gamma\left( \mathbb E_{\rho^{\pi^*}}\left[ \sum_{t=1}^\infty (\beta \gamma)^{t-1} (V^*(s_t) - h(s_t)) \right] +  \mathbb E_{\rho^{\pi}}\left[ \sum_{t=1}^\infty \gamma^{t-1} (h(s_t) - \widetilde{V}^*(s_t)) \right]\right)
    \end{flalign}
\end{lemma}
\noindent where $\rho^\pi$ denotes the trajectory \textbf{distribution} of $s_0, a_0, s_1, ...$ induced by running $\pi$ starting from $s_0$. 

This upper bound elegantly illustrates a trade-off: if the heuristic $h$ is set too high (overestimation error), it reduces the first term but increases the second term. Conversely, if $h$ is set too low (underestimation error), it increases the first term but reduces the second term. Just by inspection, it is not immediately clear whether overestimation or underestimation results in a better upper bound, as both impact the bias term in opposing ways.

We now prove that underestimation of $h$ (pessimistic $h$) is better than overestimation (i.e., $\exists s_t \in \mathcal S, h(s_t) > V^*(s_t)$) for minimizing bias. 
\begin{proof}

    \begin{enumerate}
        \item \textbf{Define Terms}: 
        \begin{itemize}
            \item Let $B_1 = \mathbb E_{\rho^{\pi^*}}\left[ \sum_{t=1}^\infty (\beta \gamma)^{t-1} (V^*(s_t) - h(s_t)) \right]$
            \item Let $B_2 = \mathbb E_{\rho^{\pi}}\left[ \sum_{t=1}^\infty \gamma^{t-1} (h(s_t) - \widetilde{V}^*(s_t)) \right]$
        \end{itemize}
        \item \textbf{Underestimation Case (Pessimistic $h$)}:
        \begin{lemma}
            \label{lemma:pessimistic_apply_all_reshaped_mdp}
            If $h$ is pessimistic with respect to $\mathcal M$, $\forall \beta \in [0,1], s\in \mathcal S, \widetilde{V}^*(s) \geq h(s)$
        \end{lemma}
        \begin{proof}
            To begin with, we need to use another lemma from \citep{DBLP:conf/nips/ChengKS21}, shown as follows:
            \begin{lemma}[Bellman equation of reshaped MDP \citep{DBLP:conf/nips/ChengKS21}]
            \label{lemma:bellman_backup_of_reshaped_M}
                For any policy $\pi$, we have
                \begin{equation}
                    \widetilde{V}^\pi(s_0) - h(s_0) = \frac{1}{1-\beta\gamma} \mathbb E_{\widetilde{d}_{s_0}^{\pi}}[ (\widetilde{\mathcal B} h)(s,a) - h(s)]
                \end{equation}
                where $\widetilde{d}_{s_0}^{\pi}$ refers to the \textbf{discounted average state distribution} of policy $\pi$ in reshaped MDP $\widetilde{\mathcal M}$. It can be expressed as $\widetilde{d}_{s_0}^{\pi} = (1 -\widetilde{\gamma}) \sum_{t=0}^\infty \widetilde{\gamma}^t d_t^{\pi}$, and $d_t^{\pi}$ is the state distribution of policy $\pi$ at time $t$ with $d_0^{\pi} = \mathbf 1\{s = s_0\}$.
            \end{lemma}
            For brevity, we do not include the proof for Lemma~\ref{lemma:bellman_backup_of_reshaped_M}.

            First of all, due to the definition of optimal value $V^*$, we have
            \begin{equation}
                \widetilde{V}^*(s_0) \geq \widetilde{V}^\pi(s_0)
            \end{equation}
            Then, according to Lemma~\ref{lemma:bellman_backup_of_reshaped_M}, we can get 
            \begin{equation}
                \widetilde{V}^*(s_0) \geq \widetilde{V}^\pi(s_0) = h(s_0) + \frac{1}{1-\beta\gamma} \mathbb E_{\widetilde{d}_{s_0}^{\pi}}[ (\widetilde{\mathcal B} h)(s,a) - h(s)]
            \end{equation}

            Recall that the use of VLM-based reward model transform the original MDP into a shaped MDP \(\widetilde{\mathcal{M}} = \langle \mathcal{S}, \mathcal{A}, \mathcal{P}, s_0, \tilde{r}, \gamma \rangle\), where the reward function becomes \(\tilde{r} = r^{e}(s, a) + r^{v}(\tau_t, l_{m(t)})\). Also, it is important to distinguish the Bellman backup equation under the two different MDPs: the original $\mathcal M$ and the reshaped $\widetilde{\mathcal M}$. By definition, $(\mathcal B h)(s,a) = r(s,a) +\gamma \mathbb E_{s' \sim \mathcal P(\cdot | s,a)}[h(s')]$. In contrast, $(\widetilde{\mathcal B} h)(s,a) = \widetilde r(s,a) + \widetilde \gamma \mathbb E_{s' \sim \mathcal P(\cdot | s,a)}[h(s')]$. Nevertheless, the two Bellman backup equations possess a remarkable property -- they are essentially equivalent to each other:
 
            \begin{flalign}
                (\widetilde{\mathcal B} h)(s,a) &= \widetilde r(s,a) + \widetilde \gamma \mathbb E_{s' \sim \mathcal P(\cdot | s,a)}[h(s')] \nonumber \\
                &= \left( r(s,a) + (1-\beta) \gamma \mathbb E_{s' \sim \mathcal P(\cdot | s ,a)}[h(s')] \right) + \beta \gamma \mathbb E_{s' \sim \mathcal P(\cdot | s,a)}[h(s')] \nonumber\\ 
                &= r(s,a) +\gamma \mathbb E_{s' \sim \mathcal P(\cdot | s,a)}[h(s')] \nonumber \\
                &= (\mathcal B h)(s,a)      \label{eq:bh_same_as_bh}
            \end{flalign}

            Let $\pi$ denote the greedy policy of $\arg \max_a (\mathcal B h)(s,a)$ and then trace back to Equation~\ref{eq:bh_same_as_bh}, we have $(\widetilde{\mathcal B} h)(s,a) = (\mathcal B h)(s,a)$, that means
            \begin{flalign}
                \widetilde{V}^*(s_0) \geq \widetilde{V}^\pi(s_0) &= h(s_0) + \frac{1}{1-\beta\gamma} \mathbb E_{\widetilde{d}_{s_0}^{\pi}}[ (\widetilde{\mathcal B} h)(s,a) - h(s)] \\
                &= h(s_0) + \frac{1}{1-\beta\gamma} \mathbb E_{\widetilde{d}_{s_0}^{\pi}}[ (\mathcal B h)(s,a) - h(s)] \\
                &\geq h(s_0) \quad \text{[direct result of $h$ being pessimistic]}
            \end{flalign}
            \textbf{Generalization to Any State}: The lemma from \citep{DBLP:conf/nips/ChengKS21} (Lemma~\ref{lemma:bellman_backup_of_reshaped_M}) is stated in terms of any starting state $s_0$. Therefore, we can replace $s_0$ with any state $s \in \mathcal{S}$ in our analysis.
            Thus we get $\forall s \in \mathcal S, \widetilde{V}^*(s) \geq h(s)$
        \end{proof}
        Lemma~\ref{lemma:pessimistic_apply_all_reshaped_mdp} implies that when $h$ is pessimistic, $B_2 = \mathbb E_{\rho^{\pi}}\left[ \sum_{t=1}^\infty \gamma^{t-1} (h(s_t) - \widetilde{V}^*(s_t)) \right] \leq 0$ because $\widetilde{V}^*(s) \geq h(s)$ for all $s$. Therefore, the bias error is only bounded by $B_1$ when $h$ is pessimistic, i.e., $\mathrm{Bias}(h,\beta,\pi) \leq B_1$. 

        \item {\textbf{Overestimation Case ($\exists s_t \in \mathcal S, h(s_t) > V^*(s_t)$)}}:
        
        We show briefly that, unlike the underestimation error, the overestimation error does not have a closed-form upper bound expression. The difficulty, as pointed out by \citet{DBLP:conf/nips/ChengKS21}, originates from the trajectory-dependence on $\mathbb{E}_{\rho^\pi}[\cdot]$ as the $l_\infty$ approximation error here can be difficult to control in large state spaces. It is possible that, upon picking up falsely high $h$ states, $\rho^\pi$ gets further distorted away from $\rho^{\pi^*}$ and accumulates more falsely high $h$ states. In other words, this trajectory dependence makes $B_2$ prone to a feedback loop of accumulating overestimation errors and resulting in a much larger upper bound for the bias.
    \end{enumerate}
\end{proof}

\subsection{Convergence Guarantee of HuRL Compared With Potential-Based Reward Shaping}
\label{app_subsec:convergene_guarantee_of_hurl}
Previous work on potential-based reward shaping (PBRS) has established that learning a policy on a reshaped MDP can converge to the optimum of the original MDP, as proved by \citet{icml/NgHR99}. They proved that the optimal $Q$ value of the reshaped MDP is equivalent to the optimal $Q$ value of the original MDP minus a state-dependent function. Consequently, for an optimal policy defined as $\pi^* = \arg \max_{a \in \mathcal{A}} Q^*(s,a)$, this additional state-dependent value does not alter the policy. Thus, the optimal policy of the original MDP $\mathcal M$ is also the optimal policy for the reshaped $\widetilde{\mathcal M}$. However, the theoretical convergence for the reshaped MDP within the framework of Heuristic-Guided Reinforcement Learning (HuRL) has not been proven with the same rigor.

The authors of HuRL made a trick which involves manipulating the coefficient $\beta$, which scales the heuristic-based auxiliary rewards. If $\beta$ increases gradually from 0 to 1 over the training process, the agent effectively transitions from interacting with the modified MDP back to the original MDP. This method ensures that, at the end of training, the agent is exactly solving the original problem. Moreover, according to the Blackwell optimal property \citep{chen2018improving}, convergence can occur before $\beta$ reaches 1, thereby allowing the trained policy to maintain optimality in the original MDP under HuRL. 

However, in the original HuRL paper, among the five test environments, the $\beta$ hyperparameter was actually fixed for four of them. Surprisingly, HuRL still converged to the optimal policy faster than the original MDP without updating $\beta$. This observation poses an enigma within the HuRL framework, as it is unclear why this convergence to the optimum of the original MDP occurs without dynamically adjusting $\beta$. 

Thus, while the practical success of HuRL in these environments is evident, the underlying reasons for such convergence remain to be fully understood. This discrepancy between theoretical expectation and empirical results leaves room for future research to explore why HuRL converges under fixed $\beta$. Nevertheless, proving the convergence properties goes beyond the scope of our paper, which aims to highlight the prevalence of false positive rewards and their impact. An alternative perspective on this issue is to consider the learning process of maximizing auxiliary rewards as akin to performing behavior cloning from experts. In this sense, we do not need to focus on convergence within the current MDP but can view it as a form of pretraining.

\section{Additional Experimental Setup Details}
\label{app:additional_experiment_details}

\subsection{VLM-based Reward Model Procedure}
\label{app_subsec:instruction_guide_procedure}
We provide a brief overview of the VLM-based reward model procedure. For a detailed procedure, please refer to \aref{ireward_procedure}. The VLM-based reward model will have a pointer to the sequence of the instruction sentence, starting at the first sentence. For original models \emph{Pixl2R} and \emph{ELLM-}, we follow the setting in their original work where for each instruction sentence (the full instruction essay will be split into multiple sentences and treat each sentence as atomic instruction $l_k$), the reward model will have a maximum cap of rewards (2.0) it can assign to the agent in one episode. When the cap is reached, the reward model will move its pointer to the next instruction sentence. This cap value was selected as a hyperparameter in prior work, intended to serve as a heuristic threshold indicating instruction completion. However, we note that this approach does not guarantee semantic completion, as high cumulative similarity scores can still occur despite incomplete or incorrect execution. For the \textsc{BiMI} reward model, the reward model will move its pointer to the next instruction sentence when the binary signal is triggered. 

In terms of the compute resources, we mainly use one NVIDIA 3090 GPU (24GB VRAM) for training and running both the VLM-based reward model and the RL policy model.

\subsection{Finetuning VLM-based Reward Models}
In contrast to previous work where they rely on hand-crafted oracle multimodal reward models (e.g., \citet{DBLP:conf/icml/DuWWCDA0A23}), we use actual pretrained VLMs to generate reward signals. 2 VLM backbone models are used in our experiments: 1) \emph{CLIP} \citep{DBLP:conf/icml/RadfordKHRGASAM21}, pretrained by image-text pairs; and (2) \emph{X-CLIP} \citep{Ma2022XCLIPEM}, pretrained by video-text pairs. In particular, \emph{Pixl2R} uses \emph{CLIP} because it only uses the single latest frame as input. In contrast, \emph{ELLM-} takes a slice of trajectory (i.e., multiple frames) as input, and thus uses either \emph{X-CLIP} or \emph{CLIP} with additional RNN encoder as the reward model.

Due to the cartoonish and abstract visuals of the testing environments, we further fine-tune the VLMs to adapt to this new visual domain. This fine-tuning is based on the \emph{contrastive learning} method. To collect data, we use well-trained expert agents from \citep{moon2023ad} to generate expert trajectories for the Crafter environments and annotate them with instructions using internal information from the game engine. For Minigrid environments, we use classical search-based planning robots to generate expert trajectories and annotate them with the corresponding task instructions. For Montezuma's Revenge, we manually annotate the expert trajectories.

For Minigrid and Crafter, we collected approximately 80,000 training pairs, each consisting of a natural language instruction and a corresponding visual observation (i.e. a short trajectory snippet). These annotations were generated using internal game state information with corresponding instruction templates. For Montezuma's Revenge, we manually created around 300 high-quality instruction-observation pairs based on expert demonstrations. These training data are of high quality, as we have made every effort to avoid false positive rewards due to poor training data quality. However, despite the fine-tuning process, false positive rewards remain unavoidable. \cref{fig:training_data_example} presents an example of instruction data employed for VLM finetuning in the Montezuma's Revenge environment.

\begin{figure}[H]
    \centering
    \includegraphics[width=0.7\columnwidth]{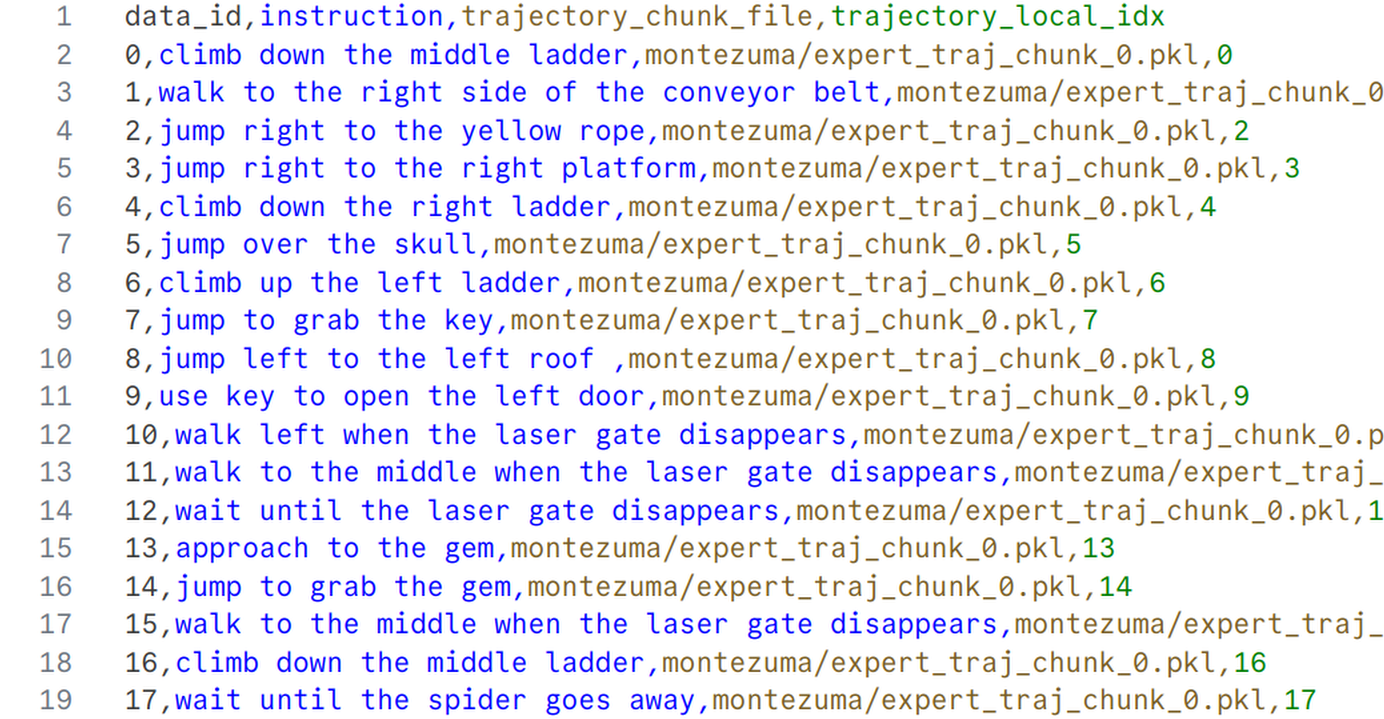}
    \caption{Example of expert instruction data for VLM finetuning under in Montezuma's Revenge environment.}
    \label{fig:training_data_example}
\end{figure}

We determine the threshold $\hat{q}$ using conformal prediction to convert continuous cosine similarity scores into binary labels for evaluating the VLM's ability to recognize correct instruction-observation pairs. This allows us to compute standard classification metrics such as precision, recall, and F1-score—providing a preliminary assessment of VLM performance independent of downstream RL interaction. The performance of the fine-tuned VLM-based reward models is shown in Table~\ref{tab:vlm_performance}. However, this high precision does not guarantee robust agent performance. As we demonstrate later, agents guided by these models significantly underperform in out-of-distribution (O.O.D.) testing environments, where false positive rewards become prevalent due to distributional shifts, revealing a critical limitation in generalization.

\begin{table}[H]
    \centering
    \caption[
        Performance of fine-tuned VLM Reward Model
    ]{Performance of fine-tuned VLM reward model on the testing dataset using the 90th percentile empirical quantile as threshold}
    \label{tab:vlm_performance}
    % \resizebox{0.7\columnwidth}{!}{%
    \begin{tabular}{@{}llllll@{}}
    \toprule
    Environment & Precision & Accuracy & F1 Score & Recall & Model       \\ \midrule
    Crafter     & 0.9847    & 0.9466   & 0.8538   & 0.9702 & CLIP ELLM-  \\
    Crafter     & 0.9799    & 0.9028   & 0.7618   & 0.9842 & CLIP Pixl2R \\
    Crafter     & 0.2095    & 0.2514   & 0.2868   & 0.9657 & XCLIP ELLM- \\ \midrule
    Minigrid    & 0.7260    & 0.9200   & 0.7849   & 0.9763 & CLIP ELLM-  \\
    Minigrid    & 0.6992    & 0.9086   & 0.7592   & 0.9616 & CLIP Pixl2R \\
    Minigrid    & 0.1716    & 0.2310   & 0.2642   & 0.9704 & XCLIP ELLM- \\ \midrule
    Montezuma   & 0.8838    & 0.9638   & 0.8825   & 0.9478 & CLIP ELLM-  \\
    Montezuma   & 0.8343    & 0.9108   & 0.7652   & 0.9842 & CLIP Pixl2R \\
    Montezuma   & 0.8044    & 0.9259   & 0.8045   & 0.9657 & XCLIP ELLM- \\ \bottomrule
    \end{tabular}%
    % }
\end{table}

\subsection{Hyperparameters for VLM Reward Model + RL Agents}
In the experiments, all methods are implemented based on PPO with same model architecture. The Minigrid and Crafter environments use the same training hyperparameters as the Achievement Distillation paper \citep{moon2023ad}. For Montezuma's Revenge, we found that the performance of the agent was sensitive to the gamma and GAE lambda parameters. To improve the performance of agents in Montezuma's Revenge, we took two additional steps: (1) normalizing the observation inputs when computing the rewards, and (2) not normalizing the advantage during the GAE calculation. The hyperparameters are shown in Table~\ref{tab:ppo_model_parameter}, Table~\ref{tab:crafter_minigrid_rl_parameter} and Table~\ref{tab:monte_rl_params}.

\begin{figure}[H]
    \footnotesize
    \centering
    \captionof{table}{Agent Policy Model Parameters}
    \label{tab:ppo_model_parameter}
    \resizebox{0.25\linewidth}{!}{%
    \begin{tabular}{@{}ll@{}}
        \toprule
    Parameter                 & Value              \\ \midrule
    model\_cls                & Recurrent PPO      \\
    hidsize                   & 1024               \\
    gru\_layers               & 1                  \\
    impala\_kwargs            &                    \\
    \quad - chans                   & {[}64, 128, 128{]} \\
    \quad - outsize                 & 256                \\
    \quad - nblock                  & 2                  \\
    \quad - post\_pool\_groups      & 1                  \\
    \quad - init\_norm\_kwargs      &                    \\
    \quad - batch\_norm             & false              \\
    \quad - group\_norm\_groups     & 1                  \\
    dense\_init\_norm\_kwargs &                    \\
    \quad - layer\_norm             & true              \\ \bottomrule
    \end{tabular}%
    }
\end{figure}

\begin{figure}[H]
    \centering
\begin{minipage}{0.46\textwidth}
    \footnotesize
    \centering
    \captionof{table}{Crafter and Minigrid RL Training Parameters}
    \label{tab:crafter_minigrid_rl_parameter}
    \resizebox{0.5\textwidth}{!}{%
    \begin{tabular}{@{}ll@{}}
        \toprule
    Parameter         & Value          \\ \midrule
    gamma              & 0.95  \\
    gae\_lambda        & 0.65  \\
    algorithm\_cls    & PPO Algorithm \\
    algorithm\_kwargs &                \\
    \quad - ppo\_nepoch     & 3              \\
    \quad - ppo\_nbatch     & 8              \\
    \quad - clip\_param     & 0.2            \\
    \quad - vf\_loss\_coef  & 0.5            \\
    \quad - ent\_coef       & 0.01           \\
    \quad - lr              & 3.0e-4         \\
    \quad - max\_grad\_norm & 0.5            \\
    \quad - aux\_freq       & 8              \\
    \quad - aux\_nepoch     & 6              \\
    \quad - pi\_dist\_coef  & 1.0            \\
    \quad - vf\_dist\_coef  & 1.0           \\ \bottomrule
    \end{tabular}%
    }
\end{minipage}
\hfill
\begin{minipage}{0.46\textwidth}
    \footnotesize
    \centering
    \captionof{table}{Montezuma's Revenge RL Training Parameters}
    \label{tab:monte_rl_params}
    \resizebox{0.5\textwidth}{!}{%
    \begin{tabular}{@{}ll@{}}
    \toprule
    Parameter             & Value          \\ \midrule
    gamma                 & 0.99           \\
    gae\_lambda           & 0.95           \\
    int\_rew\_type        & ``rnd''          \\
    pre\_obs\_norm\_steps & 50             \\
    algorithm\_cls        & PPO Algorithm \\
    algorithm\_kwargs     &                \\
    \quad - update\_proportion  & 0.25           \\
    \quad - ppo\_nepoch         & 3              \\
    \quad - ppo\_batch\_size    & 256            \\
    \quad - clip\_param         & 0.1            \\
    \quad - vf\_loss\_coef      & 0.5            \\
    \quad - ent\_coef           & 0.001          \\
    \quad - lr                  & 1.0e-4        \\ \bottomrule
    \end{tabular}%
    }
\end{minipage}
\end{figure}

\subsection{Evaluation Metrics}

As detailed in \aref{subsubsec:lit_rev_evaluation_metrics_benchmarks_rl_metrics}, we adopted the \textbf{\emph{score}} metric from the Crafter benchmark \cite{DBLP:conf/iclr/Hafner22} for performance evaluation, as it effectively measures consistent performance across multiple subtasks in sparse reward environments. Unlike the \emph{maximum total rewards} metric, which does not adequately reflect consistent performance, the score metric offers a more reliable indicator of learning progress.

\section{Details of Showing the Prevalence of False Positives in VLM Cosine Similarity Scores}

\label{app_subsec:other_envs_prevalence_of_false_pos}
In this work, we have shown that the VLM-based reward models assign high rewards to manipulated trajectory-instruction pairs, indicating the prevalence of false positive rewards. We list the figures of reward signals from learned VLMs for different types of trajectory-instruction pairs for each individual environment --- the Montezuma, Minigrid, and Crafter environments. All figures show that the VLM-based reward models assign high rewards to manipulated trajectory-instruction pairs, indicating the prevalence of false positive rewards.

\begin{figure}[H]
    \centering
    \includegraphics[width=\columnwidth]{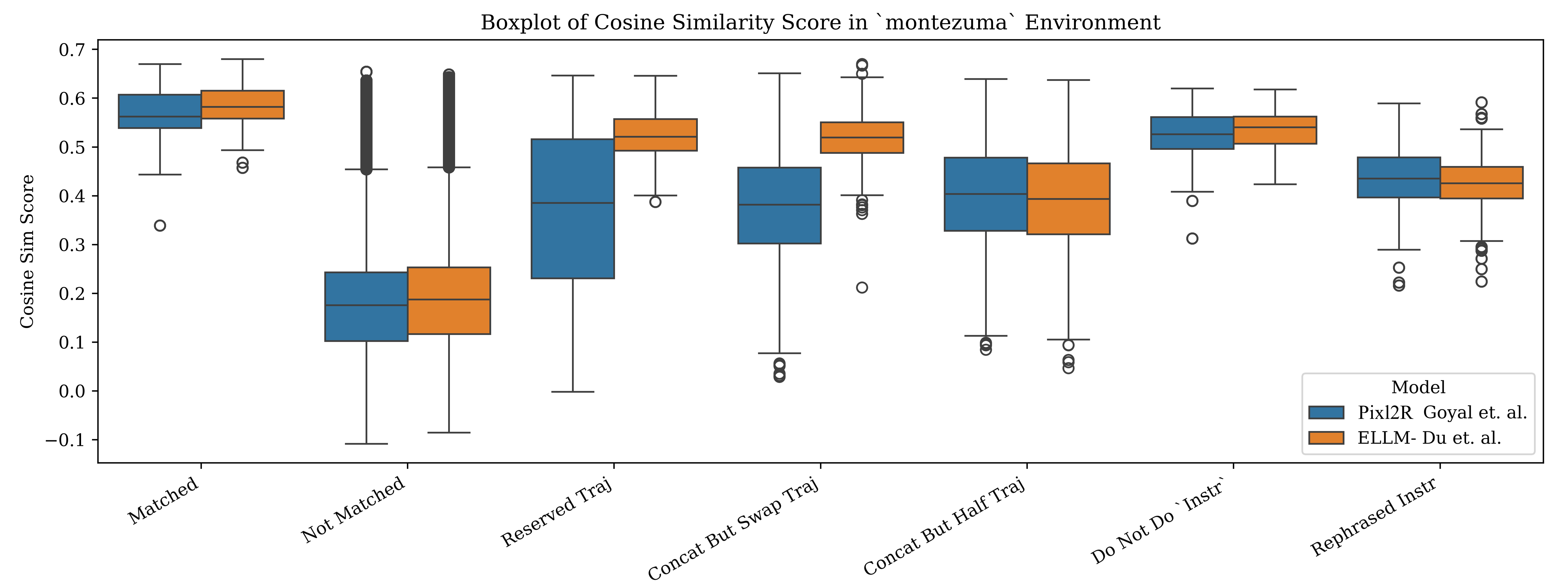}
    \caption{Cosine similarity scores for match, mismatch and manipulated trajectory-instruction pairs in Montezuma.}
\end{figure}

\begin{figure}[H]
    \centering
    \includegraphics[width=\columnwidth]{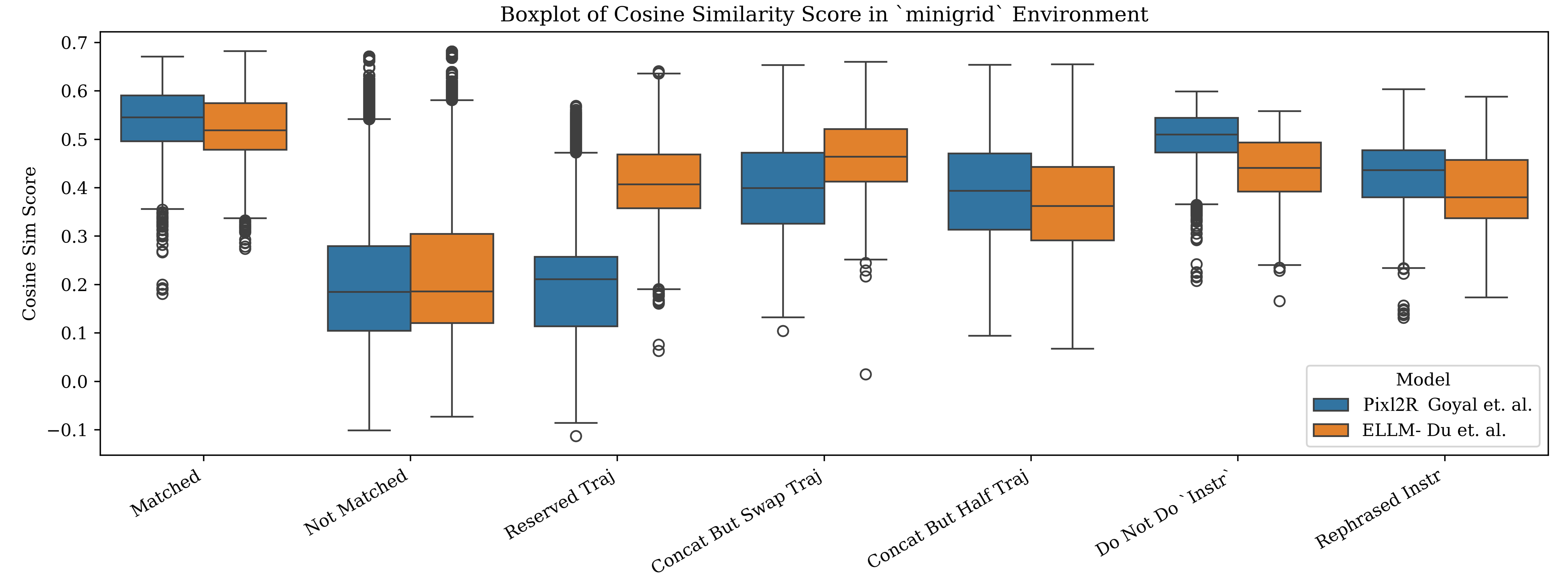}
    \caption{Cosine similarity scores for match, mismatch and manipulated trajectory-instruction pairs in Minigrid.}
\end{figure}

\begin{figure}[H]
    \centering
    \includegraphics[width=\columnwidth]{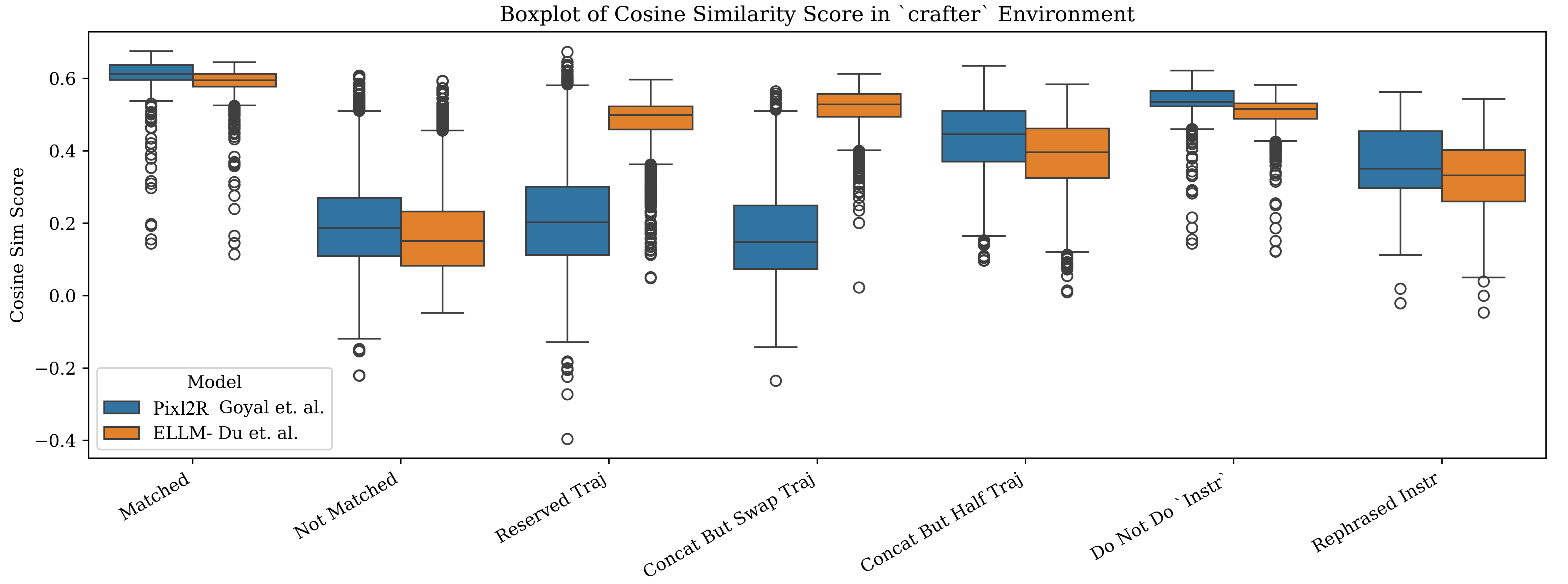}
    \caption{Cosine similarity scores for match, mismatch and manipulated trajectory-instruction pairs in Crafter.}
\end{figure}

\section{Extra Experiment on Simulated Oracle Reward Model to demonstrate \textbf{H3} and \textbf{H4}}
\label{app_subsec:simulated_oracle_reward_model}

In our efforts to assess the impact of false positive rewards from auxiliary reward model without the interference of other factors such as domain shift, poor data quality, and errors from other issues such as the choice of multimodal architectures, we devised a \textbf{simulated} auxiliary reward model (also known as \emph{oracle} auxiliary reward model) that access to internal state information from the game engine. The model compares the sequence of past actions and states of the agent with predefined target intermediate state sets that map to each instruction sentence. This is feasible in Montezuma's Revenge environment as we are able to access coordinate system information directly from the game engine. This access allows us to locate the current positional information of the agent and also label specific intermediate states as targets and assign rewards to the agent accordingly. 

We therefore designed three types of simulated reward model:

\begin{itemize}
    \item \textbf{Sim RM 1 (Perfect)} generates rewards accurately whenever the agent reaches the designated intermediate states. Furthermore, it strictly adheres to the chronological sequence of instructions; rewards for subsequent instructions are only awarded if all preceding instructions have been fulfilled.
    \item \textbf{Sim RM 2 (False Positive)} introduces a tolerance for false positive rewards but with reduced reward magnitudes, all while maintaining the temporal sequence. This is implemented by defining a radius $\sigma$, where if the agent enters a circle centered at the target state with a radius of $\sigma$, it receives a small amount of rewards.
    \item \textbf{Sim RM 3 (Temporal Insensitive False Positive)} disregards the chronological order of instructions, allowing rewards for later tasks even if earlier ones remain unfulfilled. However, note that fulfilling every sub-task will still result in the agent receiving the maximum total rewards. Therefore, in theory, the policy will eventually converge.
\end{itemize}    

Results are reported in Table~\ref{tab:chap_3_sup_oracle_RM_mainresult}. Several important observations are as follows:
\begin{table}[H]
    \footnotesize
\centering
\caption[
    Agent performance in Montezuma's Revenge with simulated (oracle) reward models
]{Agent performance in Montezuma's Revenge, evaluated across three rooms with the goal state being the exit through a designated door. Metrics measured are the Area Under the Curve (AUC) for total reward (where higher is better) and success rate (SR) of reaching the goal state. The baseline is annotated using a $\star$ symbol. The results are averaged over 5 runs, with standard deviations reported.}
\label{tab:chap_3_sup_oracle_RM_mainresult}
\begin{tblr}{width = 0.95\linewidth,colspec = {Q[473]Q[146]Q[106]},row{1} = {p},hline{1-2,8} = {-}{},hline{4,6} = {-}{dashed}}
Model                    & AUC             & SR    \\
(1) PPO \citep{DBLP:journals/corr/SchulmanWDRK17} & all failed      & 0\%   \\
(2) PPO+RND $\star$ \citep{Burda2018ExplorationBR}             & 0.550$\pm$0.066 & 100\% \\
(3) PPO + \textbf{Sim RM 1 (Perfect)}     & 0.287$\pm$0.048 & 100\% \\
(4) PPO+RND + \textbf{Sim RM 1 (Perfect)} & 0.608$\pm$0.073 & 100\% \\
(5) PPO+RND + \textbf{Sim RM 2 (False Positive)} & 0.183$\pm$0.187 & 73.3\%  \\
(6) PPO+RND + \textbf{Sim RM 3 (Temp. Insen.)} & 0.051$\pm$0.116 & 16.7\%   
\end{tblr}
\end{table}

\begin{enumerate}
    \item Perfect auxiliary reward model have shown enhanced performance compared to weak RL baselines like PPO. While agents trained solely with PPO struggled to play Montezuma, incorporating Perfect auxiliary reward model into PPO (i.e., entry (3) in Table~\ref{tab:chap_3_sup_oracle_RM_mainresult}) did find the goal state, with the success rate increasing from 0\% to 100\%. However, we observed that the auxiliary reward model learns more slowly than the intrinsic reward model. This is evidenced by its lower AUC score (0.287 for entry (3) vs. 0.550 for entry (2) in Table~\ref{tab:chap_3_sup_oracle_RM_mainresult}), where AUC reflects how quickly the agent learn to reach the goal state. This finding was initially surprising, but upon examining the agent movement heatmaps in Figure~\ref{fig:heatmapa} and Figure~\ref{fig:heatmapb}, an explanation becomes clear: the PPO+RND model discovers shorter paths to the goal state. The heatmap reveals that the PPO+RND agent learns to directly jump to a rope, bypassing the use of a ladder and conveyor belt as suggested by the expert instructions. This observation highlights a nuance in VLM-based reward signals, potentially overlooking more efficient routes when expert instructions do not serve as landmarks (a well-known concept in heuristic search in the field of classical planning).
    \item We observed a remarkable synergy between the instruction-based reward model and the intrinsic reward model. As shown in entry (4) in Table~\ref{tab:chap_3_sup_oracle_RM_mainresult}, the PPO+RND + Sim RM 1 (Perfect) agent achieves the highest AUC score, demonstrating that it learns to reach the goal state faster than any other model.
    \item False positive rewards significantly impeded learning, as shown by entries (5) and (6) in \cref{tab:chap_3_sup_oracle_RM_mainresult}. Heatmaps (\cref{fig:heatmapb}, \cref{fig:heatmapc}) reveal agents frequently visiting dead ends (e.g., falling off cliffs) under these conditions, suggesting that the agent is trapped in local minima. To assess their impact, we compare these to PPO+RND (entry (2)), a baseline lacking VLM-based instruction rewards. Since (2) does not even have a VLM-based reward model, we treat it as a baseline with \textbf{full false negative} instruction-based rewards. The AUC scores of (5) and (6) are lower than (2), indicating that false positives are more detrimental than false negatives, supporting our hypothesis \textbf{H4}, as detailed shortly.
    \item When the VLM reward model ignored temporal ordering (entry (6)), convergence failed, with the success rate dropping to 16.7\% and RND unable to recover within the time limit. Heatmap (\cref{fig:heatmapd}) illustrates the agent fixating on the final instruction (``walk left to the door'') without securing the key, trapping it in local minima near $s_0$ due to false positives --- a form of composition insensitivity (\cref{sec:approximation_error}). Moreover, given that Montezuma's Revenge environment has a fixed initial state distribution, (i.e., the agent always starts at a fixed $s_0$), it is not guaranteed that the RL algorithm will eventually reach the global optimum with maximum total rewards, as highlighted by \citet{Agarwal2019OnTT}.
\end{enumerate}

\begin{figure}[t]
    \captionsetup[sub]{font=tiny}
    \centering
     \begin{subfigure}[b]{0.24\textwidth}
         \centering
         \includegraphics[width=\textwidth, height=3.8cm]{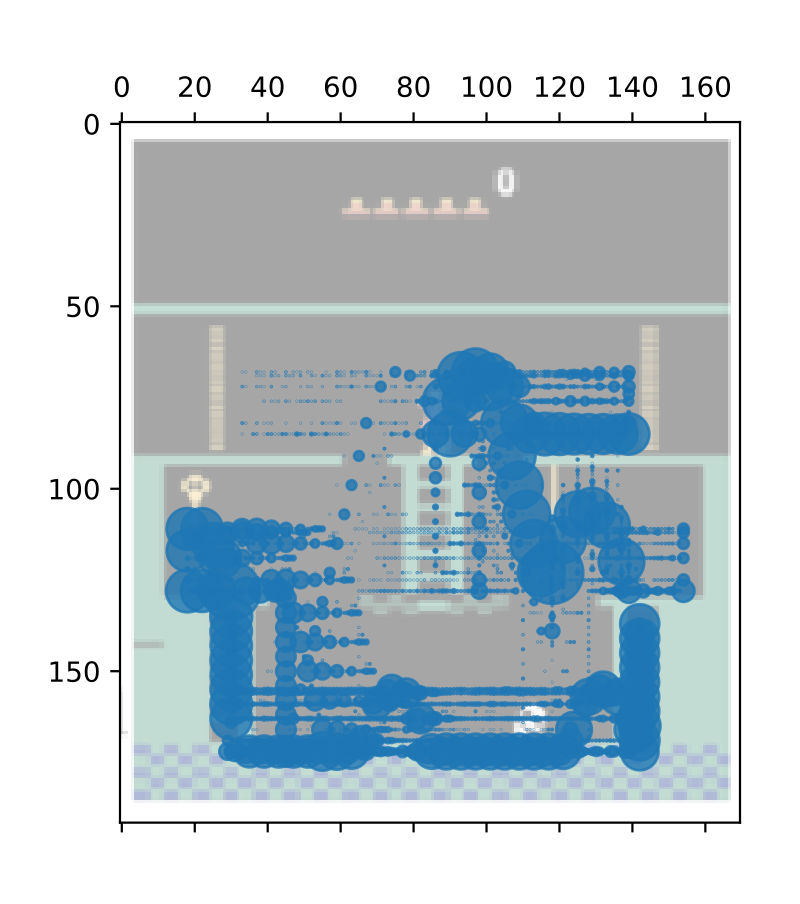}
         \caption{PPO+RND (same as + RM (False Negative))}
         \label{fig:heatmapa} 
     \end{subfigure}
     \hfill
     \begin{subfigure}[b]{0.24\textwidth}
         \centering
         \includegraphics[width=\textwidth, height=3.8cm]{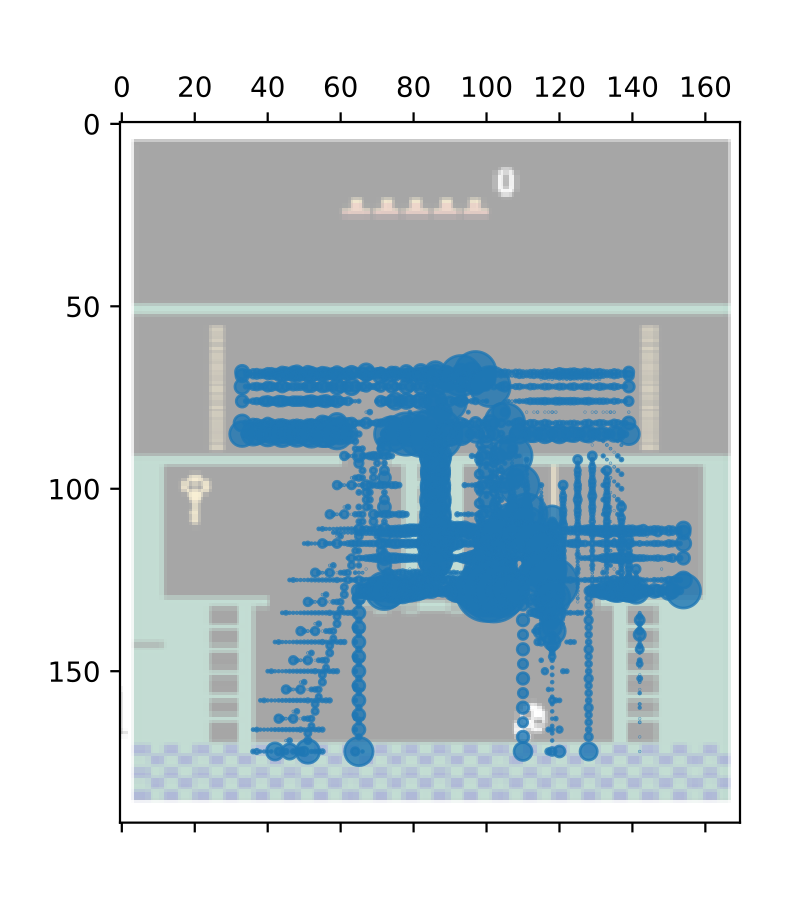}
         \caption{PPO+RND + Sim RM 1 (Perfect Simulation)}
         \label{fig:heatmapb} 
     \end{subfigure}
     \hfill
     \begin{subfigure}[b]{0.24\textwidth}
         \centering
         \includegraphics[width=\textwidth, height=3.8cm]{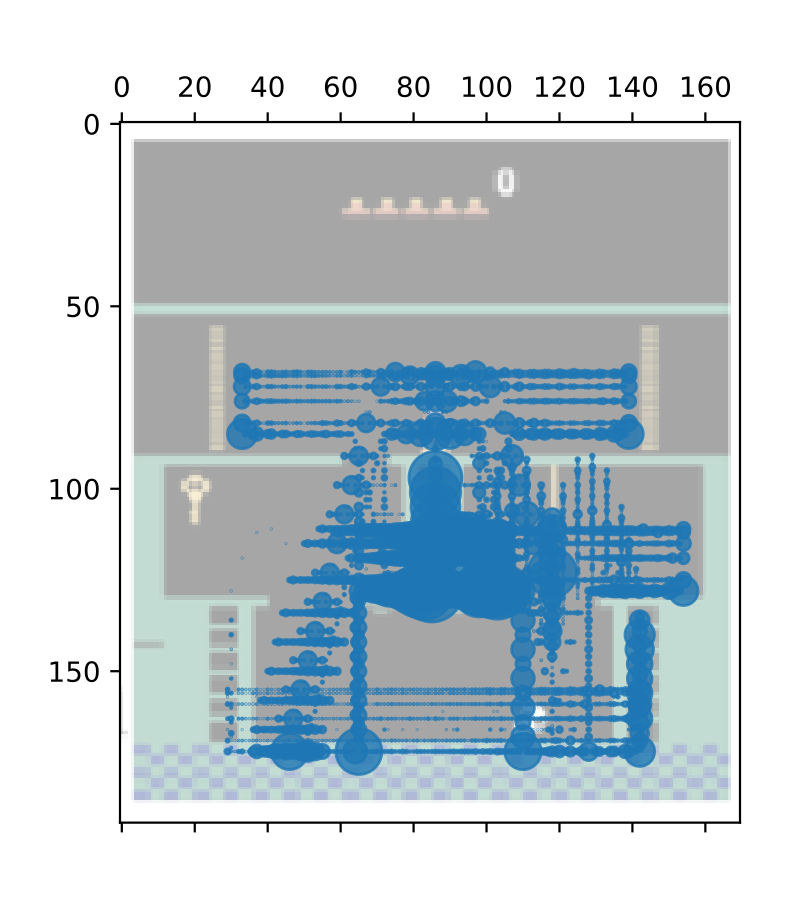}
         \caption{PPO+RND + Sim RM 2 (False Positive)}
         \label{fig:heatmapc} 
     \end{subfigure}
     \hfill
     \begin{subfigure}[b]{0.24\textwidth}
         \centering
         \includegraphics[width=\textwidth, height=3.8cm]{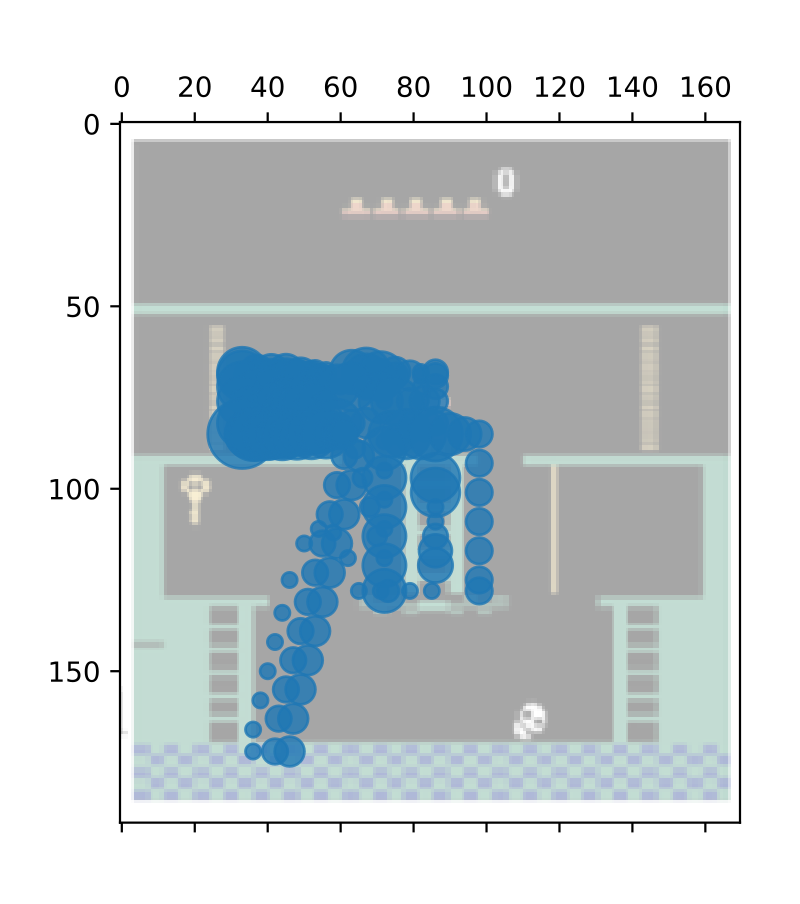}
         \caption{PPO+RND + Sim RM 3 (Temporal Insensitivity)}
        \label{fig:heatmapd}   
     \end{subfigure}

    \caption[
        Movement heatmap for various PPO+RND agents in Montezuma's Revenge
    ]{Movement heatmap for PPO+RND agents when different simulated auxiliary instruction-following-based reward models are involved.}
    \label{fig:heatmap}
\end{figure}

Figure~\ref{fig:heatmap} provides initial evidence of the differential impact of false negatives versus false positives on training outcomes, as posited by \textbf{H4}, suggesting that false positives --- particularly those tied to temporal insensitivity --- more severely degrade final scores.

\section{Algorithm for Empirical Quantile Calculation}
\label{app_subsec:empirical_quantile_calculation}
The empirical quantile $\hat{q}$ is calculated using the conformal prediction method, shown as follows:

\begin{algorithm}[H]
    \footnotesize
    \caption{Calculate Empirical Quantile ($\hat{q}$) in VLM RM}
    \label{alg:chapt_3_threshold_calculation}
    \begin{algorithmic}[1] %[1] enables line number
        \Require Calibration set $\{\tau, l\}_n$, where $l$ is the instruction sentence, $\tau$ is the corresponding trajectory, and $n$ is the number of samples;

        User-defined error rate $\alpha$; 

        VLM model reward model $v$

        \LComment{Obtain the similarity-based score}
        \State \{$r\}_n \gets \{v(\tau, l)\}_n$
        \LComment{Compute the quantile level}
        \State $q_\text{level} \gets \frac{\lceil (n-1) \times (1-\alpha) \rceil}{n}$
        \LComment{Compute the empirical quantile}
        \State $\hat q \gets \texttt{np.quantile($\{r\}_n$, $q_\text{level}$, method=`lower')}$ 
        \State \Return $\hat q$
    \end{algorithmic}
\end{algorithm}

\section{Ablation Study Lineplots for BiMI reward}
\label{app_subsec:ablation_study_lineplots_bimi}

\begin{figure}[H]
    \centering
    \includegraphics[width=0.85\textwidth]{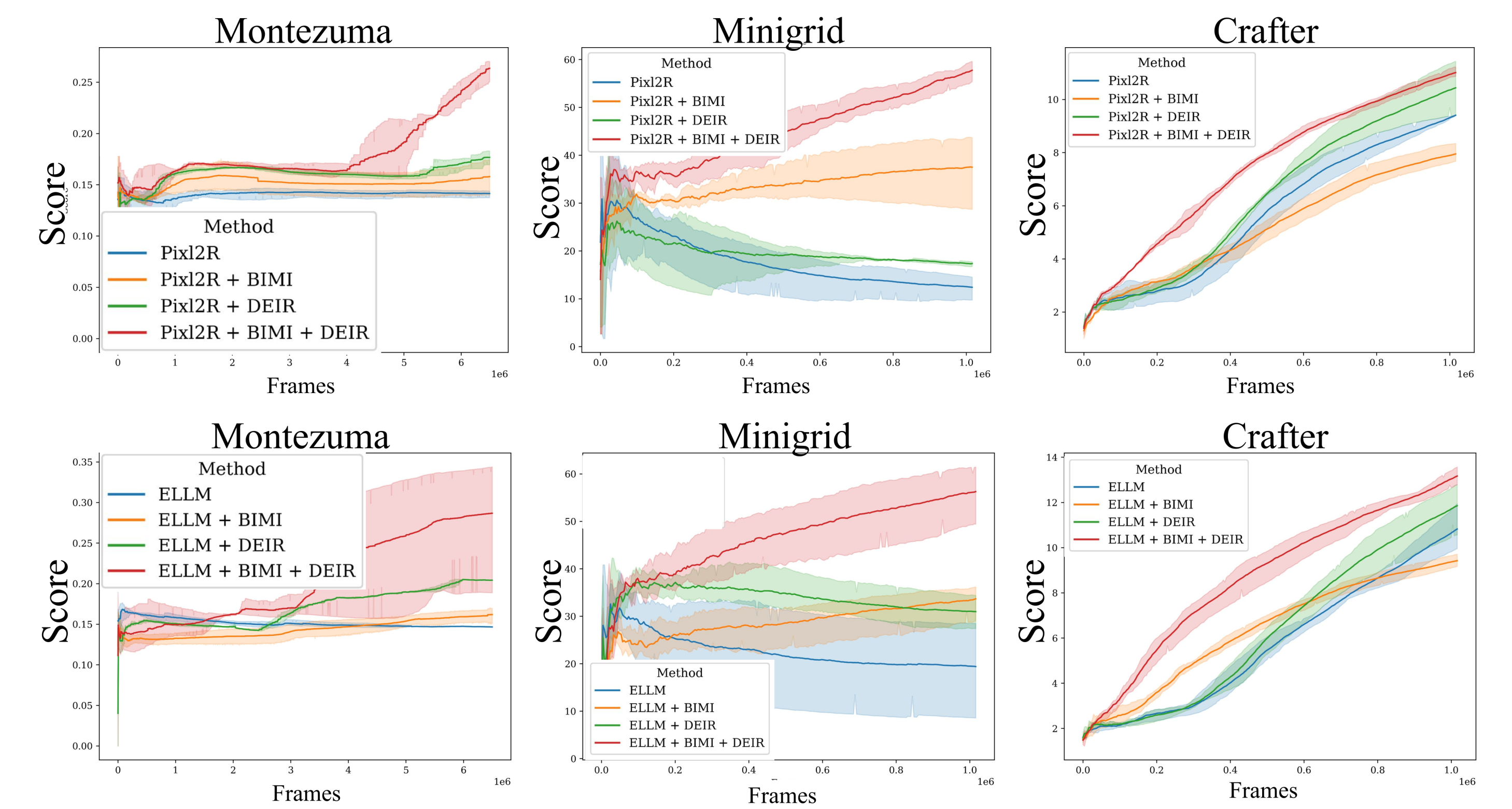}
    \caption[
        Training lineplot of \textsc{BiMI} + intrinsic reward function
    ]{Besides the improvements of the score performance of agents across different environments with the \textsc{BiMI} reward function, it also collaborates well with intrinsic rewards. Combining both can lead to significant performance improvements}
    \label{fig:main_result_lineplot_score}
\end{figure}

\begin{figure}[H]
    \centering
    \includegraphics[width=0.85\textwidth]{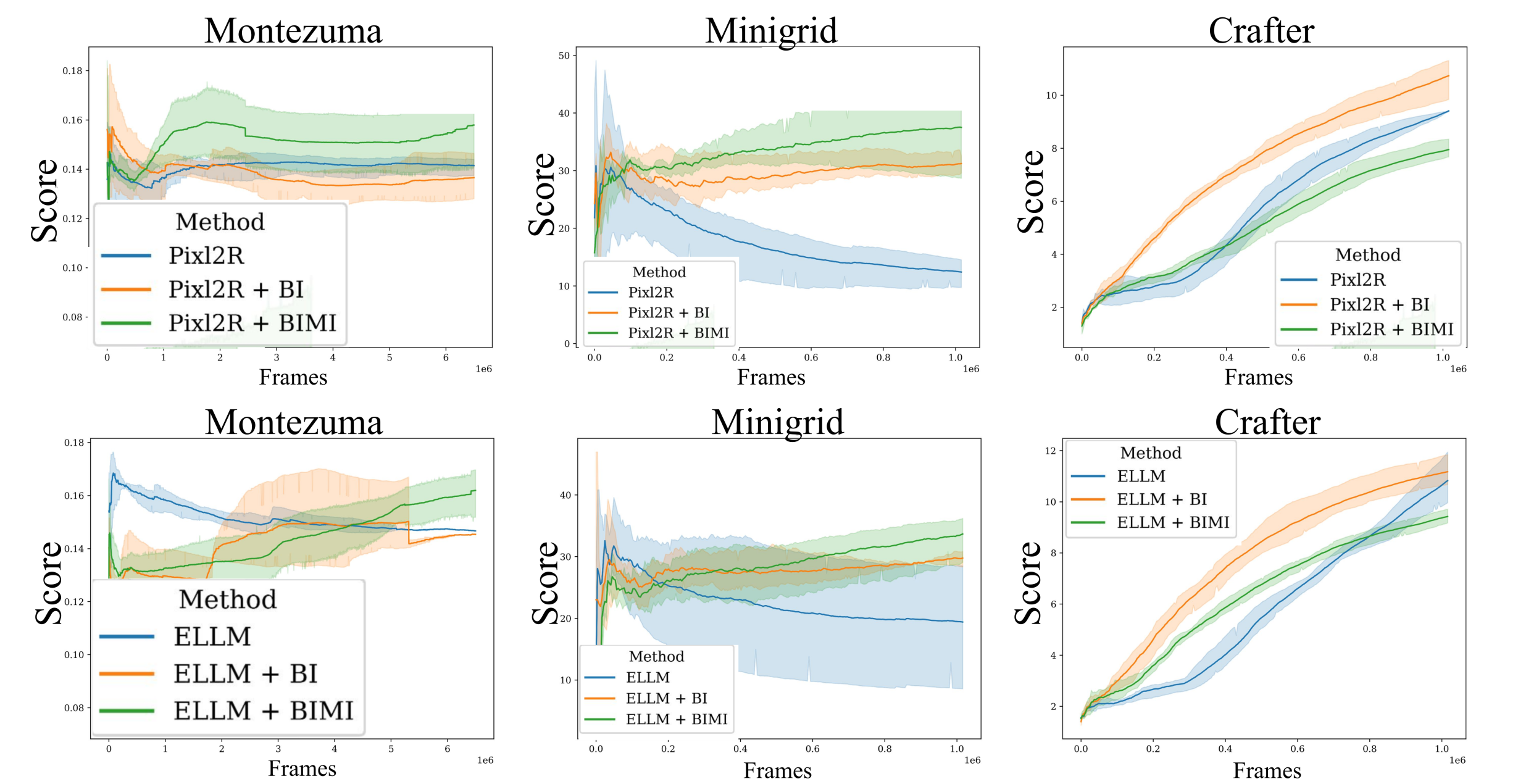}
    \caption[
        Training lineplot of ablation study on \textsc{BiMI} reward function
    ]{Ablation on the components of \textsc{BiMI} reward function. The binary reward (Bi) alone led to a 36.5\% improvement compared to original models. Excluding Crafter, Mutual Information (MI) provided a 23\% further improvement over Bi alone}
    \label{fig:ablation_bi_bimi_lineplot_score}
\end{figure}

\section{Encoder-based Vision-Language Models Explanation}
\label{subsubsec:lit_rev_architecture_training_VLM_architecture}

Similar to LLMs, VLMs can be broadly classified into two categories: \emph{encoder-based models} (or \emph{dual-encoder models}) and \emph{generative-based models} (often referred to as \emph{Vision LLMs}).

\emph{Encoder-based models} focus on mapping vision and language data into a shared semantic space where embeddings of corresponding visual and textual content are positioned close to each other. CLIP \citep{DBLP:conf/icml/RadfordKHRGASAM21} and SigLIP \citep{tschannen2025siglip2multilingualvisionlanguage} exemplify this approach. They are often referred to as VLM dual-encoders as such VLMs consists of both vision encoder and text encoder which can be operated separately to encode vision or language data. Dual-encoder models are particularly well-suited for cross-modal retrieval tasks, such as finding the most relevant image for a given text query or vice versa. This is achieved by computing similarity scores (e.g., cosine similarity) or other distance metrics between the embeddings of the query and the images in the dataset. Dual-encoders are therefore being explored to serve as learned-based reward models for embodied RL agents \citep{DBLP:conf/iclr/RocamondeMNPL24,DBLP:conf/icml/DuWWCDA0A23}, which we discuss in \aref{subsec:lit_rev_model_free_rl_LLM_VLM_instruction_reward}.

In contrast, \emph{generative-based models} such as LLaVA \citep{DBLP:conf/nips/LiuLWL23a} and Sa2VA \citep{yuan2025sa2vamarryingsam2llava} are typically built by connecting a vision encoder (e.g., CLIP-based vision encoder) to an LLM via a projection layer (e.g., a Multilayer Perceptron (MLP)).
The backbone LLM takes both vision feature vectors and text tokens as input and generates textual outputs, enabling these models to perform tasks such as Vision Question Answering (VQA). Although both categories are commonly classified as VLMs, they serve fundamentally different purposes and exhibit distinct behaviors.

\section{Sparse Reward and Random Walk}
\label{app_subsec:app_sec:sparse_reward_random_walk}

Below, we will discuss the characteristics of RL training under sparse rewards

\begin{restatable}{proposition}{propSparseReward}
    \label{prop:sparse_reward_means_random_walk}
    With sparse rewards, the gradient landscape is nearly flat,
making gradient-ascent updates indistinguishable from a random walk in parameter space.
\end{restatable}

While this may seem obvious, no existing textbook has yet explained this concept clearly. Here, we aim to formalize and justify this proposition, as it will be crucial for analyzing the performance of VLM-based reward models in this work.

Recall the gradient of the policy gradient update in online Actor Critic methods: $\theta \leftarrow \theta + \alpha \mathbb E_{\pi_{\theta}}\left[Q_\phi(s, a) \nabla_\theta \log \pi_\theta(a | s)\right]$. Using chain rule, the gradient of the policy gradient update can be decomposed as: $$\mathbb E_{\pi_{\theta}}\left[\frac{Q_\phi(s, a)}{\pi_\theta(a | s)} \nabla_\theta \pi_\theta(a | s)\right]$$

In sparse reward settings, the Q-function \( Q^{\pi_\theta}(s, a) \), often approximated by a neural network in actor-critic methods, exhibits significant randomness. Initially, the critic's weights are randomly initialized (e.g., via Xavier initialization \citep{DBLP:journals/jmlr/GlorotB10}), producing arbitrary Q-values \citep{nature/MnihKSRVBGRFOPB15}. With most transitions yielding \( r = 0 \), the TD error \( \delta = \gamma Q(s', a') - Q(s, a) \) is typically non-zero due to the stochastic variation between \( Q(s, a) \) and \( Q(s', a') \) resulting from random initialization. This leads to random updates to the Q-network, as the noisy TD error lacks a consistent direction, further increasing the stochasticity of Q-value estimates. 

Such stochasticity propagates to the policy gradient update, where it even amplifies this effect: At initialization, \( \pi_\theta(a|s) \approx 1/K \) for a discrete action space with \( K \) actions, as random weights yield a near-uniform softmax distribution. The gradient term \( \frac{\nabla_\theta \pi_\theta(a|s)}{\pi_\theta(a|s)} \) scales \( Q_\phi(s, a) \) by the inverse of the policy probability. Therefore, the update step \( \theta \leftarrow \theta + \alpha_\theta \frac{Q_\phi(s, a) \nabla_\theta \pi_\theta(a|s)}{\pi_\theta(a|s)} \) has a high variance due to noisy Q-value estimates and non-trivial gradient magnitudes from the scaling effect of $\pi_\theta(a|s)$. Consequently, the update exhibits \( \text{Var}(\hat{\nabla}_\theta J(\theta)) > 0 \) and \( \mathbb{E}[\hat{\nabla}_\theta J(\theta)] \approx 0 \). 

This behavior renders policy updates indistinguishable from a \emph{random walk} in parameter space. The parameters \( \theta \) drift stochastically, as each step's direction is dominated by noise rather than progress toward a better policy. This random walk behavior is a key reason why training under sparse rewards is challenging.

\end{document}